\RequirePackage{etoolbox}

\documentclass[ba,noinfoline]{imsart}
\pubyear{0000}
\volume{00}
\issue{0}
\doi{0000}
\firstpage{1}
\lastpage{1}

\usepackage{amsthm}
\usepackage{amsmath}
\usepackage{natbib}
\usepackage[colorlinks,citecolor=blue,urlcolor=blue,filecolor=blue,backref=page]{hyperref}
\usepackage{graphicx}

\startlocaldefs

\usepackage{amssymb}
\usepackage{caption}
\usepackage{subcaption}

\usepackage[utf8]{inputenc}
\DeclareGraphicsExtensions{.pdf}

\DeclareMathAlphabet{\pazocal}{OMS}{zplm}{m}{n}
\DeclareMathOperator*{\argmin}{arg\,min}
\DeclareMathOperator*{\argmax}{arg\,max}

\numberwithin{equation}{section}

\theoremstyle{plain}

\newtheorem{mylem}{Lemma}
\newtheorem{mydef}{Definition}
\newtheorem{mythm}{Theorem}
\endlocaldefs

\begin{document}


\begin{frontmatter}
\title{Theory of Optimal Bayesian Feature Filtering}

\runtitle{Theory of Optimal Bayesian Feature Filtering}

\begin{aug}

\author{\fnms{Ali} \snm{Foroughi pour}\thanksref{addr1}\ead[label=e1]{foroughipour.1@osu.edu}}
\and
\author{\fnms{Lori A.} \snm{Dalton}\thanksref{addr1}\ead[label=e2]{dalton@ece.osu.edu}}

\runauthor{A. Foroughi pour and L. A. Dalton}

\address[addr1]{Department	of Electrical and Computer Engineering, The Ohio State University, Columbus, OH, 43210
	\printead{e1}
	\printead{e2}
}

\end{aug}

\begin{abstract}
Optimal Bayesian feature filtering (OBF) is a supervised screening method designed for biomarker discovery.  
In this article, we prove two major theoretical properties of OBF.
First, optimal Bayesian feature selection under a general family of Bayesian models reduces to filtering \emph{if and only if} the underlying Bayesian model assumes all features are mutually independent. Therefore, OBF is optimal if and only if one assumes all features are mutually independent, and OBF is the only filter method that is optimal under at least one model in the general Bayesian framework.  
Second, OBF under independent Gaussian models is consistent under very mild conditions, including cases where the data is non-Gaussian with correlated features.  This result provides conditions where OBF is guaranteed to identify the correct feature set given enough data, and it justifies the use of OBF in non-design settings where its assumptions are invalid.  
\end{abstract}



\end{frontmatter}

\thispagestyle{empty}


\section{Introduction}

\label{sec:1_introduction}

Biomarker discovery entails mining a small-sample high-dimensional dataset for a list of features that represent potentially interesting molecular biomarkers.  The hope is that the reported features might direct future studies~\citep{ref6} that ultimately lead to new diagnostic or prognostic tests, better treatment recommendations, or a better understanding of the regulatory mechanisms underlying the biological phenomena or disease under study~\citep{ilyin_biomarker_2004, ref18, ref17}.  

Unfortunately, discovering reliable and reproducible biomarkers has proven to be difficult~\citep{diamandis_cancer_2010}.  
One reason is that the algorithms employed (see~\cite{ilyin_biomarker_2004}, \cite{saeys_review_2007}, \cite{diamandis_cancer_2010} and~\cite{ang2016supervised} for reviews on common methods) are typically not well suited for the biomarker discovery problem.  
Univariate filter methods often exhibit quirks depending on the scoring function employed~\citep{lazar2012survey}.  For example, the popular t-test cannot detect markers based on large differences between variance alone~\citep{foroughipour2018optimal}, even though such markers may have an important role to play in the disease under study or help uncover previously unknown subclasses of the disease.  
Multivariate methods may seem to have an advantage over filters because they can account for correlations; however, rather than use this correlation information to identify \emph{all} markers that may be of interest, they tend to avoid selecting redundant markers or reward selecting smaller feature sets to simplify model construction or avoid overfitting~\citep{sima2008peaking, awada2012review, ang2016supervised, li2017feature}.  This effect is so catastrophic for biomarker discovery that univariate methods often far outperform multivariate methods~\citep{sima2006should,sima2008peaking,foroughipour2018optimal}.  

Here we examine \emph{optimal Bayesian feature filtering} (OBF), a supervised univariate filter method designed from the ground up for exploratory biomarker discovery~\citep{ref8}.  OBF assumes a finite number of classes (e.g., patients given drug A versus drug B).  Under its assumed model, OBF optimally detects and ranks the set of \emph{all} features with distributional differences between the classes.  It has been shown through simulations that OBF has competitive and robust performance across Bayesian models with block-diagonal covariances, and that it enjoys particularly excellent performance when markers are individually strong with low correlations~\citep{ref10,dalton2018heuristic}.  \cite{foroughipour2018optimal} also examined the performance of OBF when its modeling assumptions (e.g., independence, priors, and Gaussianity) are violated, provided guidance on choosing inputs and objective criteria for robust performance, and demonstrated that OBF enjoys low computation cost.  

Under Gaussian models with certain non-informative priors, OBF reduces to testing each feature separately using the test statistic studied by~\cite{NP_OBF_JP_freq} and~\cite{zhang2012exact} for the equality of two Gaussian populations.  
OBF does not use classification or regression in any part of its framework.
While variable selection methods based on classification or regression (for instance LASSO) are useful for designing predictive models~\citep{o2009review}, like most multivariate methods they are typically not suitable for biomarker discovery because their objective is model construction.  
Small sample sizes worsen the overfitting problem, often resulting in small feature sets.  
If classification is involved, error estimation bias and variance result in poor selection performance~\citep{sima2006should}. 

Bayesian variable selection methods like Bayesian LASSO~\citep{park_bayesian_2008}, 
the Bayesian extension to group LASSO by~\cite{xu2015bayesian},
and works by~\cite{lee_gene_2003} and \cite{baragatti2011bayesian} based on generalized linear models (GLMs), suffer from similar problems.
Whereas OBF places priors directly on the underlying data generation model, most priors for Bayesian variable selection, for example spike and slab priors~\citep{mitchell1988bayesian,madigan1994model,george1997approaches,ishwaran2005spike}, place uncertainty on the classification or regression parameters, which are difficult to justify, interpret and validate in practice.  
Multicollinearity can be assuaged by grouping genes, but methods by~\cite{rockova2014} and~\cite{xu2015bayesian} assume grouping information is known \emph{a priori}, which is infeasible in exploratory analysis. 
Also, in contrast with OBF, Bayesian methods often rely on computationally intensive methods like Markov-Chain Monte-Carlo (MCMC) sampling or variational inference~\citep{carbonetto2012scalable}.

Shared kernel Bayesian screening (SKBS) by~\cite{lock2015shared} is an interesting approach that assumes all feature distributions belong to a family of mixture models with $K$ components, and the objective is to test whether the classes have different weights in the mixture distribution.  Whereas OBF treats each individual feature separately, SKBS uses the same $K$ dictionary mixture components for all features and allows only the kernel weights to vary.  When sample size is small we observed better performance using small $K$, but in this case the data may not be properly modeled for all features and the detected mixture components lose interpretability.  
SKBS also uses MCMC, making it more computationally expensive than OBF. 
Bayesian non-parametric methods, for example those based on Dirichlet or Pitman-Yor processes, have also gained popularity in bioinformatics for classification, inferring gene networks, clustering, and detecting chromosomal aberrations~\citep{shahbaba2009nonlinear,libbrecht2015machine,mitra2015nonparametric,ni2017heterogeneous}. 
Spike-and-slab Dirichlet processes avoid the need to specify the number of mixtures; however, it is still difficult to specify and justify the base distribution and priors in practice.  While our focus here is on the supervised case, many works like that of~\cite{cui2012spike} focus on the unsupervised case.  Computation is also a key concern; Cui and Cui use Bayesian expectation-maximization, which is more demanding than OBF.  
\cite{holmes2015two} presents a supervised method based on P\'{o}lya trees, however, the model may require larger samples than available in a typical exploratory analysis and may be sensitive to imbalanced samples.  

Our main contributions are two-fold:  
(1) we prove optimal Bayesian feature selection under a general family of Bayesian models reduces to filtering (e.g., OBF) \emph{if and only if} the underlying Bayesian model assumes all features are independent, and 
(2) we prove OBF under independent Gaussian models is a consistent estimator of the feature set we wish to select under mild conditions, including cases where the data is non-Gaussian with correlated features.  Contribution (1) has two practical implications: OBF is the only filter method for which there exists a model in the general Bayesian framework where it is optimal, and OBF is optimal if and only if one assumes all features are independent.  Contribution (2) is of enormous importance, since it provides conditions where OBF is guaranteed to identify the correct feature set given enough data, and it justifies the use of OBF in non-design settings where its assumptions are invalid.  

We review the Bayesian model in Section~\ref{sec:model} and optimal set selection in Section~\ref{sec:set_selection}.  In Section~\ref{sec:filtering} we discuss OBF and present contribution (1) in Theorem~\ref{sec:lem_MR}, and in Section~\ref{sec:consistency} we examine consistency and present contribution (2) in Theorems~\ref{sec:mar_cor} and~\ref{thm:convergence}.  We provide a demonstration on synthetic microarray data in Section~\ref{sec:simulations}, and conclude in Section~\ref{sec:conclusion}.  We provide a demonstration on real colon cancer microarray data in Sections S2 and S3 of Supplementary Material A.  

\section{Bayesian Model}
\label{sec:model}

In Section~\ref{sec:general_Bayesian_model}, we describe the general three-level Bayesian model originally proposed in~\cite{dalton_optimal_2013}.  In Sections~\ref{sec:independent_Bayesian_model} and~\ref{sec:independent_Gaussian_model} we cover the independent case and independent Gaussian case, respectively, which are originally presented in~\cite{ref8}.  Although not covered here, an independent categorical model and several models with correlations in the general Bayesian framework have been proposed~\citep{dalton_optimal_2013,ref9,pour2017optimal,ref10}.

\subsection{General Bayesian Model}
\label{sec:general_Bayesian_model}

Consider a feature selection problem in which we are to identify all features that have distinct distributions between two classes, $y=0$ and $y=1$.  Although we consider binary labels here, the multiclass case is similar and has been characterized in~\cite{pour2017multiclass}.
Let $F$ be a set of feature indices, let each feature $f \in F$ be associated with a space, $\pazocal{X}_f$, and let $\pazocal{X}=\prod_{f \in F} \pazocal{X}_f$ be the feature space.  Typically, $\pazocal{X}_f = \mathbb{R}$ for all $f$.  We call features that we wish to select, e.g. those with distributional differences between classes, ``good features.'' When viewed as a random quantity, we denote this set by $\bar{G}$, and we denote a realization of this random set by $G$.  
Likewise, we call features that we wish not to select ``bad features,'' and denote them by $\bar{B}=F\backslash \bar{G}$ when random and $B=F\backslash G$ when fixed, where ``$\backslash$'' is the set difference.  
Conditioning on events like $\{\bar{G} = G\}$ or $\{f \in \bar{G}\}$ does not mean the set of good features is deterministic. Rather, this should be interpreted as merely a hypothesis that these events hold for the current $G \subseteq F$ or $f \in F$ under consideration. Furthermore, since $\bar{G} = G$ if and only if  $\bar{B} = B$, conditioning on the event $\{\bar{G} = G\}$ is equivalent to conditioning on the event $\{\bar{B} = B\}$.  We denote conditioning on these events by ``$|G$'' or ``$|B$'', and use these notations interchangeably throughout.  

We denote a prior on $\bar{G}$ across the power set of $F$
by $p(G) = P(\bar{G}=G)$.
Given $\bar{G}=G$, let $\theta^G_y$ denote data generation parameters of class $y \in \{0,1\}$ features in $G$, let $\theta^B$ denote data generation parameters of features in $B$, and let $\theta = \{\theta^G_0,\theta^G_1,\theta^B\}$ be the set of all data generation parameters.  
Define corresponding parameter spaces: $\Theta^G_y$, $\Theta^B$ and $\Theta=\Theta^G_0 \times \Theta^G_1 \times \Theta^B$. 
We denote a prior on $\theta$ by $p(\theta|G)$, and assume $\theta^G_0$, $\theta^G_1$ and $\theta^B$ are conditionally mutually independent, i.e., 
\begin{equation}
p(\theta|G)=p(\theta^G_0|G) p(\theta^G_1|G) p(\theta^B|B).
\label{eq:prior_theta}
\end{equation}

We assume feature selection is aided by the observation of feature-label pairs, and we denote the complete dataset, including features and labels, by $S$.  Though we assume the data is complete here, the missing data problem has been studied for special cases of this model in~\cite{pour2016optimal}.  
Let $x \in \pazocal{X}$ be a feature vector, and let $x^G$ and $x^B$ denote elements of $x$ that correspond to features in $G$ and $B$ respectively.  Given $\bar{G}=G$, parameter $\theta$ and class $y$, we also assume $x^G$ and $x^B$ are independent:
\begin{equation}
\label{eq:pdf_1}
p(x|y, \theta, G)=p(x^G|\theta^G_y)p(x^B|\theta^B).
\end{equation}
Assume the data is comprised of $n$ points with $n_y$ points in class $y$, that the label of each point is determined by a process independent of $\theta$ and $G$, and that, conditioned on the labels, sample points are independent with points belonging to the same class identically distributed.  These assumptions are true in many sampling strategies, for instance random and separate sampling.
Let $S^G_y$ and $S^B$ be the part of the data corresponding to features in $G$ from class $y$ and features in $B$ from both classes, respectively. Due to independence between $x^G$ and $x^B$ and independence between sample points,
\begin{equation}
p(S|\theta, G)
\propto p(S^G_0|\theta^G_0) p(S^G_1|\theta^G_1) p(S^B|\theta^B),
\label{eq:likelihood_S_given_theta_G}
\end{equation}
where the proportionality constant depends on the distribution of $n_y$ for the given sampling strategy, $p(S^G_y|\theta^G_y)
= \prod_{x^G \in S^G_y} p(x^G|\theta^G_y)$, and $p(S^B|\theta^B) = \prod_{x^B \in S^B} p(x^B|\theta^B)$.
Thus, $S^G_0$, $S^G_1$ and $S^B$ are mutually independent given $\theta$ and $G$.  Further, from~\eqref{eq:prior_theta} and~\eqref{eq:likelihood_S_given_theta_G}, they are also independent given only $G$, that is,
\begin{align}
p(S|G)
&= \textstyle \int_{\Theta} p(\theta|G) p(S|\theta,G) d\theta 
\propto p(S^G_0|G) p(S^G_1|G) p(S^B|B),
\label{eq:likelihood_S_given_G}
\end{align}
where for $y \in \{0, 1\}$, 
\begin{equation}
p(S^G_y|G) 
= \textstyle \int_{\Theta^G_y} p(\theta^G_y|G) p(S^G_y|\theta^G_y) d\theta^G_y , 
\quad
p(S^B|B)
= \textstyle \int_{\Theta^B} p(\theta^B|B) p(S^B|\theta^B) d\theta^B.
\label{eq:likelihood_S_parts_given_G}
\end{equation}

Let $p(G|S) = P(\bar{G}=G|S)$ be the posterior probability that the set $G$ is precisely the set of good features, given our observation of the data.
Applying Bayes' rule and~\eqref{eq:likelihood_S_given_G},
\begin{eqnarray}
\label{eq:post_int}
p(G|S) \propto p(G) p(S|G) 
\propto
p(G) p(S^G_0|G) p(S^G_1|G) p(S^B|B).
\end{eqnarray}
The marginal prior and posterior probabilities that an individual feature, $f \in F$, is in $\bar{G}$ are denoted by $\pi(f) = P(f \in \bar{G}) = \sum_{G: f \in G} p(G)$ and
\begin{equation}
\pi^*(f) = P(f \in \bar{G}|S) = \sum_{G: f \in G} p(G|S),
\label{eq:posterior_f}
\end{equation}
respectively.  
Note that $P(f \in \bar{B}) = 1 - \pi(f)$ and $P(f \in \bar{B}|S) = 1 - \pi^*(f)$.  Also,
\begin{align}
\label{eq:car_pri}
E\left( |\bar{G}| \right)
&= E\left(\sum_{f \in F} I(f \in \bar{G})\right)
= \sum_{f \in F} P(f \in \bar{G})
= \sum_{f \in F} \pi(f),
\end{align}
where $|\cdot|$ denotes cardinality for sets, and $I(q)$ is the indicator function, equal to 1 if $q$ holds and 0 otherwise.  Similarly,
$E\left( |\bar{G}| \big|S\right)
= \sum_{f \in F} \pi^*(f)$.
The expected number of good features, before and after observing data, may be found from $\pi$ and $\pi^*$, respectively.

In biomarker discovery, previously known biomarkers can be integrated into the prior to aid the discovery of new biomarkers~\citep{dalton2017integrating}.  When prior knowledge is not available, improper priors for $p(\theta|G)$ may be needed and the above derivations become invalid. To circumvent this problem we: (1) require $p(\theta|G)$ to be such that the integrals in~\eqref{eq:likelihood_S_parts_given_G} are positive and finite, (2) require $\pi(G)$ to be proper, and (3) take~\eqref{eq:likelihood_S_parts_given_G}, \eqref{eq:post_int} and~\eqref{eq:posterior_f} as definitions with the proportionality constant in~\eqref{eq:post_int} defined such that $\sum_{G: G \subseteq F} p(G|S) = 1$.  
Although improper priors are controversial, see for example marginalization paradoxes described by~\cite{dawid1973marginalization}, counterexamples discussed by~\cite{jaynes2003probability}, and discussions on the Jeffreys-Lindley paradox by~\cite{robert1993note,robert2014jeffreys}, this guarantees the posterior $p(G|S)$ and marginal posterior $\pi^*(f)$ under improper priors are normalizable to valid densities and have definitions consistent with proper priors.  See Sections S5 and S6 of Supplementary Material A for further discussions on improper priors.  

\subsection{Independent Bayesian Model}
\label{sec:independent_Bayesian_model}

Assume a prior $p(G)$ on $\bar{G}$ where the events $\{f \in \bar{G} \}$ are mutually independent.  Then, 
\begin{align}
\label{eq:IPS_prior}
p(G) 
&= P\left( (\cap_{g \in G} \{g \in \bar{G}\} ) \cap (\cap_{b \in B} \{b \in \bar{B}\}) \right) \nonumber \\
&= \prod_{g \in G} \pi(g) \prod_{b \in B} (1-\pi(b)).
\end{align}
To completely characterize this prior, note that one need only specify $\pi(f)$ for all $f \in F$.
Further, if $\pi(f) = p$ is constant for all $f \in F$, then $|\bar{G}|$ is binomial$(|F|, p)$.

For every $f \in F$ we assign three parameters, $\theta_0^f$, $\theta_1^f$ and $\theta^f$, with parameter spaces $\Theta_0^f$, $\Theta_1^f$ and $\Theta^f$ and densities $p(\theta^f_0)$, $p(\theta^f_1)$ and $p(\theta^f)$, respectively.
Let $\theta_y^G = \{\theta_y^f : f \in G\}$ and $\theta^B = \{\theta^f : f \in B\}$ and assume parameters of individual features are mutually independent given $\bar{G} = G$.  Then~\eqref{eq:prior_theta} becomes 
$p(\theta|G) = \prod_{g \in G} p(\theta^g_0) p(\theta^g_1) \prod_{b \in B} p(\theta^b)$.
Finally, we assume features are mutually independent given $\bar{G} = G$, $\theta$ and $y$, thus the joint density in~\eqref{eq:pdf_1} is now of the form
$p(x|y, \theta, G) = \prod_{g \in G} p(x^g|\theta^g_y) \prod_{b \in B} p(x^b|\theta^b)$,
where $p(x^g|\theta^g_y)$ and $p(x^b|\theta^b)$ are the marginals of good and bad features, respectively.

As in~\eqref{eq:post_int}, one can show
\begin{align}
\label{eq:post_int_ind_temp}
p(G|S) \propto
p(G) \prod_{g \in G} p(S_0^g|g \in \bar{G}) p(S_1^g|g \in \bar{G}) \prod_{b \in B} p(S^b|b \in \bar{B}),
\end{align}
where, as in~\eqref{eq:likelihood_S_parts_given_G},
\begin{align}
\label{eq:likelihood_S_parts_given_f}
p(S^f_y|f \in \bar{G}) &= \textstyle \int_{\Theta^f_y} p(\theta^f_y) p(S^f_y|\theta^f_y) d\theta^f_y , 
\quad
p(S^f|f \in \bar{B}) = \textstyle \int_{\Theta^f} p(\theta^f) p(S^f|\theta^f) d\theta^f.
\end{align}
Dividing the right-hand side of~\eqref{eq:post_int_ind_temp} by the constant $\prod_{f \in F} (1-\pi(f))p(S^f|f\in\bar{B})$, we have
\begin{align}
\label{eq:post_h_prod}
p(G|S)
&\propto  \prod_{g \in G} h(g), 
\end{align}
where for all $f \in F$, we define
\begin{equation}
h(f) 
= \frac{\pi(f)}{1-\pi(f)} \times \frac{p(S_0^f|f\in\bar{G}) p(S_1^f|f\in\bar{G})}
{p(S^f|f\in\bar{B})}.
\label{eq:definition_of_h}
\end{equation}
Furthermore, from~\eqref{eq:posterior_f},
\begin{eqnarray}
\label{eq:mar_deri_post}
\pi^*(f)
=
\frac{\sum_{G: f \in G} \prod_{g \in G} h(g) }{\sum_{G} \prod_{g \in G} h(g)}  
=
\frac{h(f) \sum_{G: f \not \in G} \prod_{g \in G} h(g)}{(1+h(f))\sum_{G: f \not \in G} \prod_{g \in G} h(g)} 
=
\frac{h(f)}{1+h(f)}.
\end{eqnarray}
Once $h(f)$ is found, $\pi^*(f)$ is obtained from~\eqref{eq:mar_deri_post}. 
Note that $h(f)={\pi^*(f)}/{(1-\pi^*(f))}$.  Plugging this in~\eqref{eq:post_h_prod} and normalizing by the constant $\prod_{f \in F} (1-\pi^*(f))$, we have
\begin{align}
p(G|S)
&\propto \prod_{g \in G}  \pi^*(g)  \prod_{b \in B}  (1-\pi^*(b) ).
\label{eq:pos_CIM_2_prod}
\end{align}
In fact, \eqref{eq:pos_CIM_2_prod} holds with equality, 
thus the events $\{f \in \bar{G} \}$ are mutually independent conditioned on $S$.
Just as $\pi(f)$ characterizes $p(G)$, $\pi^*(f)$ characterizes $p(G|S)$.

When $p(\theta_y^f)$ or $p(\theta^f)$ are improper, we require $\pi(f)$ to be proper, we require the integrals in~\eqref{eq:likelihood_S_parts_given_f} to be positive and finite and take these equations as definitions, and we define $\pi^*(f) = h(f)/(1+h(f))$ as in~\eqref{eq:mar_deri_post}, where $h(f)$ is defined in~\eqref{eq:definition_of_h}.  

\subsection{Independent Gaussian Model}
\label{sec:independent_Gaussian_model}

Now suppose all features are Gaussian with conjugate priors.  If $f \in \bar{G}$ then $\theta^f_y=[\mu^f_y, \sigma^f_y]$, where $\mu^f_y$ and $\sigma^f_y$ are the mean and variance of $x^f$ in class $y$, respectively.  Similarly, if $f \in \bar{B}$, then $\theta^f=[\mu^f, \sigma^f]$, where $\mu^f$ and $\sigma^f$ are the mean and variance of $x^f$.  To simplify notation, we drop the conventional square in variances, $\sigma^f_y$ and $\sigma^f$.

Assume 
$p(\theta^f_y)=p(\sigma^f_y) p(\mu^f_y|\sigma^f_y)$,
where
$p(\sigma^f_y) = A^f_y (\sigma^f_y)^{-0.5(\kappa^f_y+2)} \exp(-0.5 s^f_y/ \sigma^f_y)$, 
$p(\mu^f_y|\sigma^f_y) = B^f_y (\sigma^f_y)^{-0.5} \exp(-0.5 \nu^f_y(\mu^f_y-m^f_y)^2/\sigma^f_y)$,
and $s^f_y, \kappa^f_y, m^f_y$ and $\nu^f_y$ are real-valued hyper-parameters. For a proper prior we require $s^f_y,\kappa^f_y,\nu^f_y>0$, in which case $p(\sigma^f_y)$ is an inverse-Wishart distribution with mean $s^f_y/(\kappa^f_y-2)$ if $\kappa^f_y>2$, and $p(\mu^f_y|\sigma^f_y)$ is Gaussian with mean $m^f_y$ and variance $\sigma^f_y/\nu^f_y$. $A^f_y$ and $B^f_y$ scale the two distributions, where under a proper prior $A^f_y= (0.5 s^f_y)^{0.5 \kappa^f_y}/\Gamma(0.5 \kappa^f_y)$ and $B^f_y=(2\pi/\nu^f_y)^{-0.5}$.

The posterior, $p(\theta^f_y|S_y^f)$, is of the same form as the prior, $p(\theta^f_y)$, with updated hyper-parameters  $\kappa^{f*}_y=\kappa^f_y+n_y$, $\nu^{f*}_y=\nu^f_y+n_y$, $m^{f*}_y=(\nu^f_y m^f_y+ n_y \hat{\mu}^f_y)/(\nu^f_y+n_y)$, and
$s^{f*}_y=s^f_y+(n_y-1) \hat{\sigma}^f_y+\frac{\nu^f_y n_y}{\nu^f_y + n_y} (\hat{\mu}^f_y-m^f_y)^2$,
where $\hat{\mu}^f_y$ and $\hat{\sigma}^f_y = \sum_{x \in S_y^f} (x-\hat{\mu}^f_y)^2/(n_y-1)$ are the sample mean and unbiased sample variance, respectively, of feature $f$ points in class $y$~\citep{murphy_conjugate_2007}. 
Note that $p(S^f_y|f\in\bar{G})$ is the normalization constant in finding the posterior, $p(\theta^f_y|S^f_y)$, from the prior times likelihood, $p(\theta^f_y)p(S^f_y|\theta^f_y)$:
\begin{align}
p(S^f_y|f\in\bar{G})
&= \frac{p(\theta^f_y)p(S^f_y|\theta^f_y)}{p(\theta^f_y|S^f_y)} 
= \frac{A^f_y B^f_y \Gamma(0.5 \kappa^{f*}_y)}
{(2 \pi)^{0.5(n_y-1)} (\nu^{f*}_y)^{0.5} (0.5 s^{f*}_y)^{0.5\kappa^{f*}_y}}.
\label{eq:normalization_y}
\end{align}
\vspace{-3mm}

Moving on to bad features, we assume,
$p(\theta^f)=p(\sigma^f) p(\mu^f|\sigma^f)$,
where given real-valued hyper-parameters $s^f, \kappa^f, m^f$, and $\nu^f$, 
$p(\sigma^f) = A^f (\sigma^f)^{-0.5(\kappa^f+2)} \exp(-0.5 s^f/ \sigma^f)$ and
$p(\mu^f|\sigma^f) = B^f (\sigma^f)^{-0.5} \exp(-0.5\nu^f (\mu^f-m^f)^2/\sigma^f)$.
For a proper prior, $s^f,\kappa^f,\nu^f>0$, 
$A^f=(0.5 s^f)^{0.5 \kappa^f}/\Gamma(0.5 \kappa^f)$ and $B^f=(2\pi/\nu^f)^{-0.5}$.  The posterior has updated hyper-parameters, $\kappa^{f*}=\kappa^f+n$, $\nu^{f*}=\nu^f+n$, $m^{f*}=(\nu^f m^f+ n \hat{\mu}^f)/(\nu^f+n)$, and
$s^{f*}=s^f+(n-1) \hat{\sigma}^f+\frac{\nu^f n}{\nu^f+n} (\hat{\mu}^f-m^f)^2$,
where $\hat{\mu}^f$ and $\hat{\sigma}^f$ are the sample mean and variance, respectively, of feature $f$ points in both classes~\citep{murphy_conjugate_2007}. As in~\eqref{eq:normalization_y}, 
\begin{eqnarray}
p(S^f|f\in\bar{B}) = \frac{A^f B^f \Gamma(0.5 \kappa^{f*})}
{(2 \pi)^{0.5(n-1)} (\nu^{f*})^{0.5} (0.5 s^{f*})^{0.5\kappa^{f*}}}.
\label{eq:normalization}
\end{eqnarray}
\vspace{-3mm}

Plugging~\eqref{eq:normalization_y} and~\eqref{eq:normalization} in~\eqref{eq:definition_of_h},
\begin{align}
h(f) 
&= \frac{\pi(f)}{1-\pi(f)} 
L^f 
\bigg(\frac{2 \pi \nu^{f*}}{\nu_0^{f*} \nu_1^{f*}}\bigg)^{0.5} 
\frac{\Gamma(0.5 \kappa_0^{f*}) \Gamma(0.5 \kappa_1^{f*}) (0.5 s^{f*})^{0.5\kappa^{f*}}}
{\Gamma(0.5 \kappa^{f*}) (0.5 s_0^{f*})^{0.5\kappa_0^{f*}}(0.5 s_1^{f*})^{0.5\kappa_1^{f*}}}, 
\label{eq:definition_of_h_Gaussian}
\end{align}
where $L^f = A_0^f B_0^f A_1^f B_1^f/(A^f B^f)$.  
If $\pi(f)$, $L^f$, $\nu_y^f$, $\nu^f$, $\kappa_y^f$ and $\kappa^f$ do not depend on $f$, 
\begin{align}
h(f)
&\propto
\frac{(s^{f*})^{0.5\kappa^{f*}}}{(s_0^{f*})^{0.5\kappa_0^{f*}}(s_1^{f*})^{0.5\kappa_1^{f*}}}.
\label{eq:IGM_improper}
\end{align}
Under improper priors we require $\pi(f)$ to be proper, and to ensure~\eqref{eq:normalization_y} and~\eqref{eq:normalization} are positive and finite we require $s_0^{f*},\kappa_0^{f*},\nu_0^{f*},s_1^{f*},\kappa_1^{f*},\nu_1^{f*},s^{f*},\kappa^{f*},\nu^{f*} > 0$ for all $f \in F$.  In addition, $L^f>0$ becomes a separate parameter specified by the user.  
All theorems in this work hold under these improper priors, and set selection under proper and improper priors for the independent Gaussian case have been studied extensively in~\cite{foroughipour2018optimal}.  Following~\cite{berger1985statistical}, \cite{degroot1970optimal} and~\cite{akaike1980interpretation}, in Section S5 of Supplementary Material A we also show that $\pi^*(f)$ from these improper priors is equivalent to a limit of $\pi^*(f)$ from a sequence of proper priors.  

\section{Optimal Bayesian Feature Selection}
\label{sec:set_selection}

We define five criteria for \emph{optimal Bayesian feature selection} under the general Bayesian model: 
(1) the \emph{maximum a posteriori} (MAP) criterion selects the feature set having the highest posterior probability of being the good feature set, 
(2) \emph{constrained MAP} (CMAP) uses the MAP objective but considers only feature sets of a given size, 
(3) the \emph{minimal risk} (MR) criterion minimizes a notion of risk, with the \emph{maximum number correct} (MNC) rule being a special case that minimizes the number of mislabeled features,
(4) \emph{constrained MNC} (CMNC) uses the MNC objective but considers only feature sets of a given size, and
(5) the \emph{Neyman-Pearson} (NP) criterion maximizes the expected number of good features selected given a limited expected number of bad features selected.  
MAP was originally presented in~\cite{dalton_optimal_2013}, while MNC and an early form of CMNC constrained to selecting two features (2MNC) were originally presented in~\cite{ref7}; all of the other criteria are new.  

\subsection{Maximum \normalfont\bfseries\itshape{a Posteriori}}
The MAP feature set is the set having maximum posterior probability:
\begin{equation}
G^{MAP}=\argmax_{G \subseteq F} p(G|S).
\label{eq:MAP}
\end{equation}
We also define the CMAP feature set to be the MAP feature set under the constraint of selecting exactly $D$ features for some user-specified constant $D$:
\begin{equation}
G^{CMAP}=\argmax_{G \subseteq F : |G|=D} p(G|S).
\end{equation}
Let $\ell(G, \bar{G})$ be a \emph{loss} function in selecting $G$ when $\bar{G}$ is the true set of good features, and let $E(\ell(G, \bar{G})|S)$ be the \emph{risk} in selecting $G$.  It can be shown that the MAP feature set minimizes risk under a zero-one loss function that assigns $\ell(G, \bar{G}) = 0$ when $\bar{G} = G$ and $\ell(G, \bar{G}) = 1$ when $\bar{G} \neq G$.  Therefore, one drawback of the MAP objective is that it assigns the same loss to all incorrect feature sets, regardless of how many features are labeled incorrectly.  This is remedied by the MR objective, described in the next section.  

\subsection{Minimal Risk}

Consider the family of objective criteria with $\ell(G, \bar{G})$ of the form:
\begin{equation}
\ell(G, \bar{G})=\lambda_{GG} |G \cap \bar{G}|+ \lambda_{GB} |G \cap \bar{B}|+\lambda_{BG} |B \cap \bar{G}|+ \lambda_{BB} |B \cap \bar{B}|,
\end{equation}
where $\lambda_{GG}$, $\lambda_{GB}$, $\lambda_{BG}$, and $\lambda_{BB}$ are constants such that $\lambda_{GB}\geq \lambda_{BB}$ and $\lambda_{BG}\geq\lambda_{GG}$.  The MR feature set is defined as:
\begin{eqnarray}
G^{MR} = \argmin_{G \subseteq F} E(\ell(G, \bar{G})|S).
\end{eqnarray}
Observe that,
\begin{align}
\label{eq:RMNC_1}
E\left( |G \cap \bar{G}| \ \big|S\right)
&=  E\left( \textstyle \sum_{g \in G} I(g \in \bar{G}) | S \right) 
= \textstyle \sum_{g \in G} P\left( g \in \bar{G} |S \right) 
= \textstyle \sum_{g \in G} \pi^*(g), \\
\label{eq:bad_picked}
E\left( |G \cap \bar{B}| \  \big| S \right)  &= \textstyle \sum_{g \in G} (1-\pi^*(g)).
\end{align}
Similarly,
$E\left( |B \cap \bar{G}| \  \big| S \right)  = \sum_{b \in B} \pi^*(b)$ and
$E\left( |B \cap \bar{B}| \  \big| S \right)  = \sum_{b \in B} (1-\pi^*(b))$.
Thus, 
\begin{align}
\label{eq:RMNC_2}
E(\ell(G, \bar{G})|S)
&= \lambda_{GG} \textstyle \sum_{g \in G} \pi^*(g) +
\lambda_{GB} \textstyle \sum_{g \in G} (1-\pi^*(g)) 
\nonumber \\
&\quad 
+ \lambda_{BG} \textstyle \sum_{b \in B} \pi^*(b) +
\lambda_{BB} \textstyle \sum_{b \in B} (1-\pi^*(b)).
\end{align}
$E(\ell(G, \bar{G})|S)$ is minimized by considering each feature, $f \in F$, individually.  In particular, $f$ is in $G^{MR}$ if the risk incurred by including this feature, $\lambda_{GG} \pi^*(f) + \lambda_{GB} (1-\pi^*(f))$, is less than the risk incurred by not including it, $\lambda_{BG}\pi^*(f) + \lambda_{BB} (1-\pi^*(f))$, or equivalently, if
$(\lambda_{GB}+\lambda_{BG}-\lambda_{GG}-\lambda_{BB}) \pi^*(f) > \lambda_{GB} - \lambda_{BB}$.
Thus,
\begin{eqnarray}
\label{eq:GMR}
G^{MR} = \left\lbrace f \in F : \pi^*(f) > T \right\rbrace,
\end{eqnarray}
where $T={(\lambda_{GB}-\lambda_{BB})}/{(\lambda_{GB}+\lambda_{BG}-\lambda_{GG}-\lambda_{BB})}$.  In other words, the MR objective ranks features by their marginal posterior probability of being in $\bar{G}$, and selects those with probabilities exceeding a given threshold.

When $\lambda_{GG}=\lambda_{BB}=0$ and $\lambda_{GB}=\lambda_{BG}=1$, the MR cost function minimizes the expectation of the number of mislabeled features, $|G \cap \bar{B}| + |B \cap \bar{G}|$, 
or equivalently, maximizes the expectation of the number of correctly labeled features, $c(G, \bar{G})=|G \cap \bar{G}| + |B \cap \bar{B}|$.
This results in the MNC objective:
\begin{equation}
G^{MNC}=\argmax_{G \subseteq F} E(c(G, \bar{G})|S)
= \left\lbrace f \in F : \pi^*(f) > 0.5 \right\rbrace.
\end{equation}
MNC thus selects features with a posterior probability of being in $\bar{G}$ greater than $0.5$.  

CMR minimizes risk under the constraint of selecting exactly $D$ features:
\begin{eqnarray}
\label{eq:CMR_1}
G^{CMR} = \argmin_{G \subseteq F : |G|=D} E( \ell(G, \bar{G}) | S ).
\end{eqnarray}
Following a similar procedure used to derive~\eqref{eq:GMR}, observe:
\begin{equation}
\label{eq:CMR_2}
G^{CMR}= \argmax_{G \subseteq F : |G|=D} \sum_{g \in G} \pi^*(g).
\end{equation}
Thus, $G^{CMR}$ ranks $\pi^*(f)$ and selects the $D$ features with highest rank.  
Since the $\lambda$'s need not be specified, we also call this criterion CMNC.  

\subsection{Neyman-Pearson}
\label{app:set_selection}

Viewing the number of correctly identified good features, $|G \cap \bar{G}|$, as the number of \emph{true positives}, and the number of incorrectly identified bad features, $|G \cap \bar{B}|$, as the number of \emph{false positives}, 
the NP objective maximizes the expected number of true positives while bounding the expected number of false positives by $0 \leq \alpha \leq E( |\bar{B}| \big| S)$:
\begin{equation}
\begin{aligned}
G^{NP} = \
& \underset{G \subseteq F}{\argmax}
& & E\left( |G \cap \bar{G}| \big| S\right) \\
& \text{subject to}
& & E\left( |G \cap \bar{B}| \big| S\right)\leq \alpha.
\end{aligned}
\end{equation}
From~\eqref{eq:RMNC_1} and \eqref{eq:bad_picked}, we have that
\begin{equation}
\begin{aligned}
G^{NP} = \ 
& \underset{G \subseteq F}{\argmax}
& & \sum_{g \in G} \pi^*(g) \\
& \text{subject to}
& & \sum_{g \in G} (1-\pi^*(g)) \leq \alpha.
\end{aligned}
\end{equation}
This is solved by ranking $\pi^*(f)$ and iteratively adding features with highest rank to $G^{NP}$ until adding a new feature results in violating the constraint.  NP is closely related to MR and CMNC in that all of these methods rank features using the same scoring function, $\pi^*(f)$.  However, they use different score cutoffs; in MR the cutoff is a constant threshold, in CMNC the cutoff forces a certain set size, and in NP the cutoff depends on the values of the $\pi^*(f)$.  
For selection rule $G^{k}$ with free parameter $k$, plotting the pair $(E( |G^k \cap \bar{B}| \big| S), E( |G^k \cap \bar{G}| \big| S))$ in the 
$[0, E( |\bar{B}| \big| S)]\times[0, E( |\bar{G}| \big| S)]$ 
space under various $k$ results in a curve analogous to a \emph{receiver operating characteristic} (ROC) curve.  The ROC curve for MR (varying $T$), CMNC (varying $D$) and NP (varying $\alpha$) are all
\begin{align}
(k - \textstyle \sum_{f=1}^k \pi^*_{(f)}, \textstyle \sum_{f=1}^k \pi^*_{(f)})
\label{eq:ROC}
\end{align}
for $k = 0, 1, \ldots, |F|$, where the $\pi^*_{(f)}$ are the $\pi^*(f)$ ordered from largest to smallest.  

\section{Optimal Bayesian Feature Filtering}
\label{sec:filtering}

In the general Bayesian model, MAP and CMAP require finding $p(G|S)$ for all $G \subseteq F$, which is computationally prohibitive when $|F|$ is large. Although MR (and thus MNC), CMNC and NP always reduce to ranking features by $\pi^*(f)$ with various methods of thresholding, finding $\pi^*(f)$ also requires evaluating $p(G|S)$ for all $G \subseteq F$.  In this section, we discuss how this problem is circumvented under independent Bayesian models.  

Under independent Bayesian models, any method that ranks features by $\pi^*(f)$ (or equivalently $h(f)$) and selects top ranking features is considered an OBF rule.  
While MAP and CMAP generally do not reduce to ranking $\pi^*(f)$, in independent Bayesian models MAP reduces to MNC and CMAP reduces to CMNC by~\eqref{eq:pos_CIM_2_prod} and~\eqref{eq:MAP}, 
thus all selection criteria covered in Section~\ref{sec:set_selection} reduce to OBF rules.  
Furthermore, since $\pi^*(f)$ can be found separately for each feature under independent Bayesian models via~\eqref{eq:mar_deri_post} (for instance using~\eqref{eq:definition_of_h_Gaussian} or~\eqref{eq:IGM_improper} in the Gaussian case), all OBF rules reduce to filtering.  
The fact that optimal Bayesian feature selection reduces to filtering under independent models is not surprising, in light of similar results for Bayesian multiple comparison rules~\citep{muller2006fdr}.  
By assuming independence we lose the ability to take advantage of correlations, but we greatly simplify optimal selection.

Define a \emph{univariate filter on $F$} to be a feature selection rule that ranks features by a scoring function $h(f, S^f)$, which is a function of only the feature index $f$ and the portion of the labeled data corresponding to $f$, and selects top ranking features using some score thresholding rule, which is based on only the set of feature scores. t-tests with~\cite{benjamini1995controlling} multiple testing correction are univariate filters. Define a \emph{simple univariate filter on $F$} to be a univariate filter that uses a constant threshold, i.e., a feature selection rule that reduces to the form:
\begin{equation}
G = \{f \in F : h(f, S^f) > T\},
\end{equation}
where $T$ is a constant. t-tests without multiple testing correction are simple univariate filters. By the following theorem, not only does optimal selection reduce to OBF under independent models, but optimal selection reduces to simple univariate filtering \emph{only} under independent models, and the resulting filter must be equivalent to an OBF rule.  

\begin{mythm}
	\label{sec:lem_MR}
	MR under a general Bayesian model $\mathcal{M}$ on feature set $F$ is a simple univariate filter on $F$ for all thresholds $T$ if and only if there exists an independent Bayesian model $\mathcal{M}^{\prime}$ on $F$ such that $\pi^*(f|\mathcal{M}^{\prime}) = \pi^*(f|\mathcal{M})$ for all $f \in F$ and all labeled datasets $S$.  
\end{mythm}

\begin{proof}
	Suppose an independent Bayesian model, $\mathcal{M}^{\prime}$, exists as characterized above.  Let $T$ be an arbitrary constant.  MR simplifies to
	$G^{MR}=\{ f \in F: \pi^*(f | \mathcal{M}^{\prime}) > T \}$ by~\eqref{eq:GMR},
	where $\pi^*(f | \mathcal{M}^{\prime})$, given in~\eqref{eq:mar_deri_post}, depends only on $f$ and $S^f$ (note that $S^f$ is comprised of $S^f_0$ and $S^f_1$, along with the labels).
	Thus, MR reduces to a simple univariate filter on $F$ under both $\mathcal{M}^{\prime}$ and $\mathcal{M}$ for all $T$.
	
	Now suppose that MR under $\mathcal{M}$ is a simple univariate filter on $F$ for all $T$.  Suppose there exist samples $S_{\bullet} \neq S_{\circ}$ and $f \in F$ such that $S_{\bullet}^f = S_{\circ}^f$, but $P(f \in \bar{G}|S_{\bullet}, \mathcal{M}) > P(f \in \bar{G}|S_{\circ}, \mathcal{M})$.  Let $T$ be the midpoint between $P(f \in \bar{G}|S_{\bullet}, \mathcal{M})$ and $P(f \in \bar{G}|S_{\circ}, \mathcal{M})$.  MR at threshold $T$ selects $f$ under $S_{\bullet}$, but does not select $f$ under $S_{\circ}$.  This contradicts the premise that MR is a simple univariate filter. Thus, for all triplets $S_{\bullet}$, $S_{\circ}$ and $f$ such that $S_{\bullet} \neq S_{\circ}$ and $S_{\bullet}^{f} = S_{\circ}^{f}$, we must have $P(f \in \bar{G} | S_{\bullet}, \mathcal{M}) = P(f \in \bar{G} | S_{\circ}, \mathcal{M})$.  Fix $f_0 \in F$.  Assume that $P(f_0 \in \bar{G} | S, \mathcal{M})$, which is in general a function of $S$, cannot be written as a function of only $S^{f_0}$.  Then there exists a pair of samples $S_{\bullet}$ and $S_{\circ}$ such that $S_{\bullet} \neq S_{\circ}$, $S_{\bullet}^{f_0} = S_{\circ}^{f_0}$ and $P(f_0 \in \bar{G} | S_{\bullet}, \mathcal{M}) \neq P(f_0 \in \bar{G} | S_{\circ}, \mathcal{M})$.  By contradiction, $P(f_0 \in \bar{G} | S, \mathcal{M})$ can be written as a function of only $S^{f_0}$.  Since $f_0$ is arbitrary, we must have that the marginal posterior for each feature can be expressed as $\pi^*(f|\mathcal{M}) \equiv P(f \in \bar{G}|S, \mathcal{M}) = P(f \in \bar{G}|S^f, \mathcal{M})$ for all $f \in F$ and all $S$.
	From Bayes rule,
	\begin{align}
	\pi^*(f|\mathcal{M})
	&= \frac{p_0}
	{p_0 + p_1},
	\label{eq:general_model}
	\end{align}
	where
	$p_0 = P(f \in \bar{G}|\mathcal{M}) \prod_{y\in\{0,1\}} p(S_y^f|f \in \bar{G}, \mathcal{M})$,
	$p_1 = P(f \in \bar{B}|\mathcal{M}) p(S^f|f \in \bar{B}, \mathcal{M})$, 
	\begin{align}
	p(S_y^g|g \in \bar{G}, \mathcal{M})
	&= \sum_{G:g \not \in G} P(\bar{G} = G \cup \{g\} |g \in \bar{G}, \mathcal{M}) p(S_y^g|G \cup \{g\}, \mathcal{M}), \\
	p(S^b|b \in \bar{B}, \mathcal{M})
	&= \sum_{B:b \not \in B} P(\bar{B} = B \cup \{b\}|b \in \bar{B}, \mathcal{M}) p(S^b|B \cup \{b\}, \mathcal{M}),
	\end{align}
	$p(S_y^g|G, \mathcal{M})
	= \int_{\Theta_y^{G}}
	p(\theta_y^{G}|G, \mathcal{M}) p(S_y^g|\theta_y^{G}, \mathcal{M})
	d\theta_y^{G}$
	and
	$p(S^b|B, \mathcal{M})
	= \int_{\Theta^{B}}
	p(\theta^{B}|B, \mathcal{M}) 
	\linebreak 
	p(S^b|\theta^{B}, \mathcal{M})
	d\theta^{B}$.
	We now construct an independent Bayesian model, $\mathcal{M}^{\prime}$.  The idea is to create auxiliary random variables for each $f \in F$ that are independent from other features and yet sufficient to describe $\pi^*(f|\mathcal{M})$.  
	Define $P(f \in \bar{G}|\mathcal{M}^{\prime}) = P(f \in \bar{G}|\mathcal{M})$, define the data generation parameters $\phi_y^g = \{\bar{H}, \theta_y^{\bar{H} \cup \{g\}}\}$ for each $g \in F$, and define priors on a realization of $H \subseteq F \backslash\{g\}$ and $\theta_y^{H \cup \{g\}} \in \Theta_y^{H \cup \{g\}}$ from $\mathcal{M}$ by,
	\begin{align}
	p(\phi_y^g|\mathcal{M}^{\prime}) 
	= P(\bar{G} = H \cup \{g\}|g \in \bar{G}, \mathcal{M}) p(\theta_y^{H \cup \{g\}}|H \cup \{g\}, \mathcal{M}).
	\end{align}
	Similarly, for all $b \in F$, define $\phi^b = \{\bar{H}, \theta^{\bar{H} \cup \{b\}}\}$, and define priors on $H \subseteq F \backslash\{b\}$ and $\theta^{H \cup \{b\}} \in \Theta^{H \cup \{b\}}$ from $\mathcal{M}$ by,
	\begin{align}
	p(\phi^b|\mathcal{M}^{\prime})
	= P(\bar{B} = H \cup \{b\}|b \in \bar{B}, \mathcal{M}) p(\theta^{H \cup \{b\}}|H \cup \{b\}, \mathcal{M}).
	\end{align}
	In this way, for each feature $f \in F$ we merge the identity of features excluding $f$ with the data generation parameters.
	Finally, we define the distributions
	$p(x^g|\phi_y^g, \mathcal{M}^{\prime})
	= p(x^g|\theta_y^{H \cup \{g\}}, \mathcal{M})$ and
	$p(x^b|\phi^b, \mathcal{M}^{\prime})
	= p(x^b|\theta^{H \cup \{b\}}, \mathcal{M})$
	using the marginal distributions of $x^f$ under $\mathcal{M}$.  Note that
	$p(S_y^g|\phi_y^g, \mathcal{M}^{\prime})
	= p(S_y^g|\theta_y^{H \cup \{g\}}, \mathcal{M})$ and
	$p(S^b|\phi^b, \mathcal{M}^{\prime})
	= p(S^b|\theta^{H \cup \{b\}}, \mathcal{M})$.
	Applying~\eqref{eq:mar_deri_post}, the definition of $h(f)$, and the
	definition of $P(f \in \bar{G}|\mathcal{M}^{\prime})$, $\pi^*(f|\mathcal{M}^{\prime})$ is of the form in~\eqref{eq:general_model} with
	$p_0 = P(f \in \bar{G}|\mathcal{M}) \prod_{y \in \{0,1\}} p(S_y^f|f\in\bar{G}, \mathcal{M}^{\prime})$ and
	$p_1 = P(f \in \bar{B}|\mathcal{M}) p(S^f|f \in \bar{B}, \mathcal{M}^{\prime})$,
	where
	\begin{align}
	p(S_y^g|g \in \bar{G}, \mathcal{M}^{\prime})
	&= \sum_{H: g \not \in H} 
	\hspace{-6.5mm}
	\int\limits_{\hspace{8mm} \Theta_y^{H \cup \{g\}}}
	\hspace{-6.5mm}
	p(\{H, \theta_y^{H \cup \{g\}}\}|\mathcal{M}^{\prime}) 
	p(S_y^g|\{H, \theta_y^{H \cup \{g\}}\}, \mathcal{M}^{\prime})
	d\theta_y^{H \cup \{g\}} \\
	p(S^b|b \in \bar{B}, \mathcal{M}^{\prime})
	&= \sum_{H: b \not \in H} 
	\hspace{-6.5mm}
	\int\limits_{\hspace{8mm} \Theta^{H \cup \{b\}}}
	\hspace{-6.5mm}
	p(\{H, \theta^{H \cup \{b\}} \}|\mathcal{M}^{\prime}) 
	p(S^b|\{H, \theta^{H \cup \{b\}} \}, \mathcal{M}^{\prime})
	d\theta^{H \cup \{b\}}.
	\end{align}
	Plugging in $p(\phi_y^g|\mathcal{M}^{\prime})$, $p(\phi^b|\mathcal{M}^{\prime})$, $p(S_y^g|\phi_y^g, \mathcal{M}^{\prime})$ and $p(S^b|\phi^b, \mathcal{M}^{\prime})$, and comparing $p(S_y^g|g \in \bar{G}, \mathcal{M}^{\prime})$ and $p(S^b|b \in \bar{B}, \mathcal{M}^{\prime})$ with counterparts in $\mathcal{M}$, we have $\pi^*(f|\mathcal{M}^{\prime}) = \pi^*(f|\mathcal{M})$.
\end{proof}

\section{Consistency}
\label{sec:consistency}

A key property of any estimator is consistency: as data are collected, will the estimator converge to the quantity it is to estimate?  We are now interested in frequentist asymptotics, that is, the behavior of an estimator under a fixed set of good features, $\bar{G}$, a fixed set of parameters, $\bar{\theta}$, and the corresponding sampling distribution.

Let $S_{\infty}$ denote a countably infinite labeled dataset, and let $S_n$ denote the first $n$ observations.
In general, a sequence of estimators, $\hat{\theta}_n(S_n)$ for $n \geq 1$, of a parameter, $\bar{\theta}$,  is said to be strongly consistent at $\bar{\theta}$ if
\begin{align}
P(\hat{\theta}_n(S_n) \to \bar{\theta} \big| \bar{\theta}) = 1,
\end{align}
where convergence is understood with respect to a distance metric $d$, and this probability is taken with respect to the infinite sampling distribution on $S_{\infty}$ under some true data generation parameter, $\bar{\theta}$.
For feature selection, we will use $d(\bar{G}, G) = I(\bar{G} \neq G)$.  Under this metric, $G_n \to \bar{G}$ if and only if $G_n = \bar{G}$ for all but finitely many $n$.  The following theorem addresses the convergence of MR, CMNC and NP under any sequence of posteriors, $p(G|S_n)$.  The posteriors may be based on any general Bayesian model.  

\begin{mythm}
	\label{sec:mar_cor}
	Fix $S_{\infty}$.  If $\lim_{n \to \infty} p(\bar{G}|S_n) = 1$, then $G^{MR} \to \bar{G}$ if $T \in (0, 1)$, $G^{CMNC} \to \bar{G}$ if $D=|\bar{G}|$, and $G^{NP} \to \bar{G}$ if $\alpha \in (0, 1)$.
\end{mythm}

\begin{proof}
	By~\eqref{eq:posterior_f}, $\lim_{n \to \infty} p(\bar{G}|S_n) = 1$ implies $\pi^*(g) \to 1$ and $\pi^*(b) \to 0$ for all $g \in \bar{G}$ and $b \in \bar{B}$. The consistency of MR and NP follow immediately for the range of $T$ and $\alpha$ specified, and the consistency of CMNC follows if $D = |\bar{G}|$.  
\end{proof}

By Theorem~\ref{sec:mar_cor}, if $p(G|S_n)$ converges \emph{almost surely} (a.s.), i.e., with probability $1$, to a point mass at $\bar{G}$, then MR (and thus MNC) and NP are strongly consistent and CMNC is strongly consistent when constrained to select the correct number of features.  In Section~\ref{sec:convergence_posterior} we prove that $p(G|S_n)$ converges almost surely for independent Gaussian models under very mild conditions; the data need not be independent or Gaussian.  

\subsection{Convergence of \texorpdfstring{$p(G|S_n)$}{p(G|S_n)} Under Independent Gaussian Models}
\label{sec:convergence_posterior}

For fixed $\bar{G}$, 
let 
$F_{\infty}^{\bar{G}}$
be the infinite sampling distribution on $S_{\infty}$.
For fixed $S_n$, define $\rho=n_0/n$, 
$c^f_y=s^{f*}_y/(n_y-1)$ for all $f \in F$ and $y =0, 1$, and $c^f=s^{f*}/(n-1)$ for all $f \in F$. 
Throughout this section, we assume $p(G|S_n)$ is calculated under an independent Gaussian model with proper or improper priors on $p(\theta_y^f)$ and $p(\theta^f)$, and (in a slight generalization) allow $p(G)$ to be arbitrary.  
Allowing $p(G)$ to be arbitrary, equations analogous to~\eqref{eq:post_h_prod} and~\eqref{eq:definition_of_h_Gaussian} are straightforward to derive.  We have:
\begin{equation}
p(G|S_n) \propto a(G, S_n) z(G, S_n), 
\end{equation}
where $z(G, S_n) = p(G) \prod_{f \in G} l(f, S_n)$, 
\begin{align}
l(f, S_n)
&= L^f(n_0, n_1)
\frac{\Gamma(0.5\kappa^{f*}_0) \Gamma(0.5\kappa^{f*}_1)}{\Gamma(0.5\kappa^{f*})}
\bigg( \frac{2 \pi \nu^{f*} 0.5^{\kappa^f - \kappa^f_0 - \kappa^f_1}(n-1)^{\kappa^{f*}}}
{\nu^{f*}_0 \nu^{f*}_1 (n_0-1)^{\kappa_0^{f*}} (n_1 - 1)^{\kappa_1^{f*}}
 }\bigg) ^{0.5}
\end{align}
and
\begin{equation}
a(G, S_n)= \prod_{f \in G} \bigg( \frac{(c^f)^{\kappa^{f*}}}{(c^f_0)^{\kappa^{f*}_0} (c^f_1)^{\kappa^{f*}_1}} \bigg)^{0.5}.
\end{equation}
We write $L^f$ as a function of $n_0$ and $n_1$ to emphasize that it may be allowed to change depending on the sample size.  We assume all other inputs and hyper-parameters of the independent Gaussian model are constant across all samples sizes.  

\begin{mydef}
	\label{def:true_distribution_constraint}
	\textit{$\bar{G}$ is an independent unambiguous set of good features if,
		for each $g \in \bar{G}$ $\mu^g_y$ and $\sigma^g_y$ exist and are finite such that either $\mu^g_0 \neq \mu^g_1$ or $\sigma^g_0 \neq \sigma^g_0$,
		and for each $b \in \bar{B}$ $\mu^b_y$ and $\sigma^b_y$ exist and are finite such that $\mu^b = \mu^b_0 = \mu^b_1$ and $\sigma^b = \sigma^b_0 = \sigma^b_0$.}
\end{mydef}

\begin{mydef}
	\label{def:sampling_strategy}
	\textit{$S_{\infty}$ is called a balanced sample if the label of sample points are such that $\liminf_{n \to \infty} \rho >0$ and $\limsup_{n \to \infty} \rho <1$, and, conditioned on the labels, sample points are independent with points belonging to the same class identically distributed.}
\end{mydef}

\begin{mydef}
	\label{def:OBF_constraint}
	\textit{$p(\theta|G)$ is called semi-proper if, for all $f \in F$, there exists $c > 0$ and $p < 1$ such that
		\begin{equation}
		L^f(n_0, n_1) \sim c n^{p}
		\end{equation}
		as $n \to \infty$. $f \sim g$ as $n \to \infty$ means $\lim_{n \to \infty} f(n)/g(n)=1$.}
\end{mydef}

The following theorem proves our desired result.  Three lemmas used in the proof are provided in Section S1 of Supplementary Material A. The conditions assumed by the theorem are very mild.  Condition (i) is based on Definition~\ref{def:true_distribution_constraint} and essentially says that $\bar{G}$ is really the feature set we wish to select, i.e., good features must truly have different means or variances between the classes, and bad features must truly have the same means and the same variances between the classes.  Conditions (i) and (ii) require certain moments to exist, but there is no requirement for the data to be Gaussian or for features to be independent from each other.  Condition (iii) is based on Definition~\ref{def:sampling_strategy} and addresses the sampling strategy; the assumptions are similar to those made by most finite sample data generation models for classification, with an additional requirement on the infinite sample that the proportion of points observed in either class must not converge to zero.  Conditions (iv) and (v) place constraints on the inputs to OBF.  Condition (iv) requires that OBF assign a non-zero probability prior to the feature set we ultimately wish to select, which is easily achieved by setting $0 < \pi(f) < 1$ for all $f \in F$.  Condition (v) is based on Definition~\ref{def:OBF_constraint} and addresses the possibility that one might input different $L^f$ for an improper prior to OBF depending on sample size.  Condition (v) is always satisfied with $p = 0$ under proper priors, and under improper priors with $L^f$ set to a positive constant across all samples sizes.  
By Theorems~\ref{sec:mar_cor} and~\ref{thm:convergence}, under these conditions and posteriors computed based on the independent Gaussian model, we have that MR (and thus MNC and MAP) is strongly consistent, and CMNC (and thus CMAP) is strongly consistent when constrained to select the correct number of features.

The proof of Theorem~\ref{thm:convergence} also characterizes the rate of convergence of the posterior.  Under the conditions stated in the theorem, there exist $R>1$ and $N>0$ such that $h(g) > R^n$ (a.s.) for all $n > N$ and all good features $g \in \bar{G}$.  Equivalently, there exist $0<r<1$ and $N>0$ such that $\pi^*(g) > 1-r^n$ (a.s.) for all $n > N$ and all $g \in \bar{G}$; thus the marginal posterior of good features converges to $1$ at least exponentially (a.s.).  Further, there exist $c, N>0$ such that $h(b) < n^{-c}$ (a.s.) for all $n > N$ and all bad features $b \in \bar{B}$.  Equivalently, there exist $c,N > 0$ such that $\pi^*(b) < n^{-c}$ (a.s.) for all $n > N$ and all $b \in \bar{B}$; thus the marginal posterior of bad features converges to $0$ at least polynomially (a.s.).  Extending these facts to the full posterior on feature sets, there exist $0 < r < 1$ and $N > 0$ such that 
\begin{align}
\frac{p(G|S_n)}{p(\bar{G}|S_n)} < r^n  \quad \mbox{a.s.}
\end{align}
for all $n > N$ and all $G$ missing at least one feature in $\bar{G}$, and there exist $c,N>0$ such that 
\begin{align}
\frac{p(G|S_n)}{p(\bar{G}|S_n)} < n^{-c}  \quad \mbox{a.s.}
\end{align}
for all $n > N$ and all $G \neq \bar{G}$. More discussions on rates of convergence are provided in Section S6 of Supplementary Material A. 

\begin{mythm}
	\label{thm:convergence}
	Suppose the following are true:
	(i) $\bar{G}$ is an independent unambiguous set of good features,
	(ii) Fourth order moments exist and are finite for all features $b \in \bar{B}$,
	(iii) $S_{\infty}$ is a balanced sample with probability $1$,
	(iv) $p(\bar{G}) \neq 0$, and
	(v) $p(\theta|G)$ is semi-proper.
	Then
	$\lim_{n \to \infty} p(\bar{G}|S_n)=1$ for $F_{\infty}^{\bar{G}}$-almost all sequences.
\end{mythm}

\begin{proof}
	It suffices to show that for all $ G \subseteq F$ such that $G \neq \bar{G}$,
	\begin{equation}
	\label{eq:hokm_thm}
	\lim_{n \to \infty} \frac{p(G|S_n)}{p(\bar{G}|S_n)}=0 \quad \mbox{a.s.}
	\end{equation}
	Let $G \neq \bar{G}$. If $p(G) = 0$, then \eqref{eq:hokm_thm} holds trivially.  Thus, assume $p(G) \neq 0$. Note that
	\begin{eqnarray}
	\frac{p(G|S_n)}{p(\bar{G}|S_n)}
	= \frac{z(G, S_n)}{z(\bar{G}, S_n)}
	\prod_{g \in  B \cap \bar{G}} \bigg(  \frac {(c^g_0)^{\kappa^{g*}_0} (c^g_1)^{\kappa^{g*}_1}} {(c^g)^{\kappa^{g*}}} \bigg)^{0.5} 
	\prod_{b \in G \cap \bar{B}} \bigg(  \frac {(c^b)^{\kappa^{b*}}}{(c^b_0)^{\kappa^{b*}_0} (c^b_1)^{\kappa^{b*}_1}} \bigg)^{0.5}.
	\end{eqnarray}
	Since $p(\theta|G)$ is semi-proper, by Lemma~S1 in Supplementary Material A, there exists $L_1>0$ and $q>0$ such that
	\begin{equation}
	\frac{z(G, S_n)}{z(\bar{G}, S_n)} \sim L_1 n^{q(|\bar{G}|-|G|)}
	\end{equation}
	as $n \to \infty$ (a.s.), where $\sim$ denotes asymptotic equivalence. Therefore, it suffices to show that for each $ g \in B \cap \bar{G}$ and each $b \in G \cap \bar{B}$ we have
	\begin{align}
	\label{eq:good_case}
	\lim_{n \to \infty}  n^{q}  \bigg(  \frac {(c^g_0)^{\kappa^{g*}_0} (c^g_1)^{\kappa^{g*}_1}} {(c^g)^{\kappa^{g*}}} \bigg)^{0.5} &= 0 \quad \mbox{a.s.}, \\
	\label{eq:bad_case}
	\lim_{n \to \infty}  n^{-q}  \bigg(  \frac {(c^b)^{\kappa^{b*}}} {(c^b_0)^{\kappa^{b*}_0} (c^b_1)^{\kappa^{b*}_1}}  \bigg)^{0.5} &= 0 \quad \mbox{a.s.}
	\end{align}
	
	First, we prove \eqref{eq:good_case}. Let $ g \in B \cap \bar{G}$. Consider a fixed sample in which $\hat{\mu}_y^g$ converges to $\mu^g$ and $\hat{\sigma}_y^g$ converges to $\sigma_y^g$ for $y = 0, 1$.  Since sample points in a class are independent and identically distributed with finite first and second order moments, this event occurs almost surely by the strong law of large numbers.
	By Lemma~S2 in Supplementary Material A, there exists $\epsilon>0$ and $L_2 > 0$ such that for $n$ large enough
	\begin{equation}
	n^{q}  \bigg(  \frac {(c^g_0)^{\kappa^{g*}_0} (c^g_1)^{\kappa^{g*}_1}} {(c^g)^{\kappa^{g*}}} \bigg)^{0.5} < n^{q} L_2 (1-\epsilon)^{0.5n}.
	\end{equation}
	Since the limit of the right-hand side is zero, so is that of left-hand side.
	
	Now we prove \eqref{eq:bad_case}. Let $b \in G \cap \bar{B}$. Observe that
	\begin{equation}
	\frac{c^b}{\hat{\sigma}^b}=1+ \frac{s^b}{(n-1)\hat{\sigma}^b}+
	\frac{\nu^b n (\hat{\mu}^b-m^b)^2}{\hat{\sigma}^b(n-1)(\nu^b+n)}.
	\end{equation}
	Consider a fixed sample in which $\hat{\mu}_y^b$ and $\hat{\mu}^b$ are bounded and $\hat{\sigma}_y^b$ and $\hat{\sigma}^b$ converge to $\sigma^b$, which occurs almost surely. There exists $L_4>0$ such that for $n$ large enough,
	\begin{eqnarray}
	\label{eq:cl_a}
	1 < \frac{c^b}{\hat{\sigma}^b} < 1+\frac{L_{4}}{n}.
	\end{eqnarray}
	Similarly, there exists $L_{50},L_{51}>0$  such that for $n$ large enough
	\begin{equation}
	\label{eq:cl_0}
	1 < \frac{c^b_0}{\hat{\sigma}^b_0} < 1+\frac{L_{50}}{n} 
	\quad \text{and} \quad
	1 < \frac{c^b_1}{\hat{\sigma}^b_1} < 1+\frac{L_{51}}{n} .
	\end{equation}
	From~\eqref{eq:cl_a} and~\eqref{eq:cl_0} we conclude there exists $L_6>0$ such that for  $n$ large enough:
	\begin{equation}
	\bigg( \frac  {c^b} {(c^b_0)^\rho (c^b_1)^{1-\rho} } \bigg)^{0.5n}  < L_6 \bigg(  \frac  {\hat{\sigma}^b} {(\hat{\sigma}^b_0)^\rho (\hat{\sigma}^b_1)^{1-\rho} }  \bigg)^{0.5n} .
	\end{equation}
	Furthermore, as $c^b$ and $c^b_y$ converge, there exists $L_7>0$ such that for $n$ large enough,
	\begin{equation}
	\bigg(  \frac  {(c^b)^{\kappa^{b}}} {(c^b_0)^{\kappa^{b}_0} (c^b_1)^{\kappa^{b}_1}} \bigg)^{0.5} < L_7.
	\end{equation}
	Therefore, for $n$ large enough we may write
	\begin{eqnarray}
	\label{eq:conv_bad_out}
	n^{-q}  \bigg(  \frac {(c^b)^{\kappa^{b*}}} {(c^b_0)^{\kappa^{b*}_0} (c^b_1)^{\kappa^{b*}_1}}  \bigg)^{0.5} <
	\frac{L_6 L_7}{n^{q}} \bigg(  \frac  {\hat{\sigma}^b} {(\hat{\sigma}^b_0)^\rho (\hat{\sigma}^b_1)^{1-\rho} }  \bigg)^{0.5n} .
	\end{eqnarray}
	The following property of sample variance holds, provided that sample moments exist:
	\begin{align}
	\label{eq:bad_var}
	\hat{\sigma}^b
	&= \rho \hat{\sigma}_0^b
	+ (1-\rho) \hat{\sigma}_1^b
	+ \frac{\rho(1-\rho) n}{n-1} (\hat{\mu}_0^b - \hat{\mu}_1^b)^2 
	- \frac{1-\rho}{n-1} \hat{\sigma}_0^b
	- \frac{\rho }{n-1} \hat{\sigma}_1^b \nonumber \\
	&\leq \rho \hat{\sigma}_0^b
	+ (1-\rho) \hat{\sigma}_1^b
	+ \frac{\rho(1-\rho) n}{n-1} (\hat{\mu}_0^b - \hat{\mu}_1^b)^2.
	\end{align}
	Let us consider the sample mean term in \eqref{eq:bad_var}.  Since $\hat{\sigma}^b_0,\hat{\sigma}^b_1 \to \sigma^b$, for $n$ large enough,
	\begin{equation}
	\label{eq:mean_term_bound}
	\frac{n\rho (1-\rho) (\hat{\mu}^b_0-\hat{\mu}^b_1)^2}{(n-1)(\hat{\sigma}^b_0)^{\rho} (\hat{\sigma}^b_1)^{1-\rho}} < \frac{ 2 \rho (1-\rho) (\hat{\mu}^b_0-\hat{\mu}^b_1)^2}{\sigma^b}.
	\end{equation}
	Recall that~\eqref{eq:cl_a} through~\eqref{eq:conv_bad_out} and \eqref{eq:mean_term_bound} hold when the sample means are bounded and the sample variances converge to $\sigma^b$.  We now consider the rate of convergence of the means and variances.  Suppose $x_i$, $i=1, \ldots, n_0$, are the values of feature $b$ for points in class 0. Observe that ${(x_i-\mu^b)}/{\sqrt{\sigma^b}}$ are independent random variables with zero mean and unit variance.
	By the law of the iterated logarithm~\citep{kolmogorov_uber_1929},
	\begin{equation}
	\limsup_{n_0 \to \infty} \bigg| (n_0 \log \log n_0)^{-0.5} \sum_{i = 1}^{n_0} \frac{x_i-\mu^b}{\sqrt{\sigma^b}}   \bigg| =\sqrt{2}  \quad \mbox{a.s.}
	\end{equation}
	Further,
	\begin{eqnarray}
	|\hat{\mu}^b_0-\mu^b|
	= \frac{\sqrt{\sigma^b}}{n_0} \bigg| \sum_{i = 1}^{n_0} \frac{x_i-\mu^b}{\sqrt{\sigma^b}}   \bigg|.
	\end{eqnarray}
	Hence for $n$ large enough,
	\begin{eqnarray}
	\label{eq:eta_0_bound}
	|\hat{\mu}^b_0-\mu^b|< 2 \sqrt{\frac{\sigma^b \log \log \rho n}{\rho n}} < 2 \sqrt{\frac{\sigma^b \log \log n}{\rho n}} \quad \mbox{a.s.}
	\end{eqnarray}
	Similarly, for $n$ large enough,
	\begin{equation}
	|\hat{\mu}^b_1-\mu^b|< 2 \sqrt{\frac{\sigma^b \log \log (1-\rho) n}{(1-\rho) n}} < 2 \sqrt{\frac{\sigma^b \log \log n}{(1-\rho) n}} \quad \mbox{a.s.}
	\end{equation}
	By the triangle inequality, for $n$ large enough,
	\begin{equation}
	\label{eq:d_bound}
	|\hat{\mu}^b_0-\hat{\mu}^b_1 |< 2 \sqrt{\sigma^b} \left(  \rho^{-0.5} + (1-\rho)^{-0.5} \right) \sqrt{\frac{\log \log n}{n}} \quad \mbox{a.s.}
	\end{equation}
	Note that for all $0 < \rho < 1$, 
	\begin{equation}
	\label{eq:rho_bound}
	\rho (1-\rho) \left( \rho^{-0.5} + (1-\rho)^{-0.5} \right) ^2 \leq 2.
	\end{equation}
	Combining \eqref{eq:mean_term_bound}, \eqref{eq:d_bound}, and \eqref{eq:rho_bound}, we see that for $n$ large enough,
	\begin{equation}
	\label{eq:mean_term_bound_2}
	\frac{ n \rho (1-\rho) (\hat{\mu}^b_0-\hat{\mu}^b_1)^2}{ (n-1) (\hat{\sigma}^b_0)^{\rho} (\hat{\sigma}^b_1)^{1-\rho}} < 16 \frac{\log \log n}{n} \quad \mbox{a.s.}
	\end{equation}
	Now, consider variance terms in \eqref{eq:bad_var}. We have another property of sample variance:
	\begin{align}
	\label{eq:sig_0_bound_0}
	|\hat{\sigma}^b_0-\sigma^b |
	&=
	\bigg| \frac{1}{n_0-1} \sum_{i=1}^{n_0} (x _i-\hat{\mu}^b_0)^2 -\sigma^b \bigg| \nonumber \\
	&=
	\bigg| \frac{1}{n_0-1} \sum_{i=1}^{n_0} (x _i-\mu^b+\mu^b-\hat{\mu}^b_0)^2 -\sigma^b \bigg| \nonumber \\
	&=
	\bigg| \frac{1}{n_0-1} \sum_{i=1}^{n_0} \left(  (x _i-\mu^b)^2-\sigma^b \right)
	-
	\frac{n_0}{n_0 -1} (\mu^b-\hat{\mu}^b_0)^2 + \frac{1}{n_0-1} \sigma^b \bigg|
	\nonumber \\
	&\leq
	\bigg| \frac{1}{n_0-1} \sum_{i=1}^{n_0} \left(  (x _i-\mu^b)^2-\sigma^b \right) \bigg| 
	+
	\frac{n_0}{n_0 -1} (\mu^b-\hat{\mu}^b_0)^2 + \frac{1}{n_0-1} \sigma^b .
	\end{align}
	Under balanced sampling, $\rho n$ increases with $n$.  Note that $ \lim_{n \to \infty} \rho n / (\rho n -1)=1$, thus $\rho n / (\rho n -1)< 2$ for $n$ large enough (a.s).  Also, $1/(\rho n -1)< (\log \log n)/(\rho n)$ for $n$ large enough (a.s). In addition, we can use \eqref{eq:eta_0_bound} to bound $(\mu^b-\hat{\mu}^b_0)^2$. Hence, for $n$ large enough, 
	\begin{equation}
	\label{eq:sig_0_bound_1}
	\frac{\rho n}{\rho n -1} (\mu^b-\hat{\mu}^b_0)^2 + \frac{1}{\rho n-1} \sigma^b  < 9 \sigma^b \frac{\log \log n}{\rho n} \quad \mbox{a.s.}
	\end{equation}
	Since we assume fourth order (and thus lower order) moments of features in $\bar{B}$ are finite, the variance of $(x_i - \mu^b)^2/\sigma^b$ is finite, and we call this variance $K_0$.  Again applying the law of the iterated logarithm,
	\begin{eqnarray}
	\label{eq:sig_0_bound_2}
	\frac{1}{\rho n-1} \sum_{i=1}^{\rho n} \bigg( \bigg( \frac{x _i-\mu^b}{\sqrt{\sigma^b}}\bigg) ^2 -1 \bigg) < 2 \sqrt{K_0 \frac{\log \log \rho n}{\rho n}} \quad \mbox{a.s.}
	\end{eqnarray}
	Combining \eqref{eq:sig_0_bound_0}, \eqref{eq:sig_0_bound_1}, and \eqref{eq:sig_0_bound_2} we conclude that for $n$ large enough,
	\begin{eqnarray}
	\label{eq:sig_0_bound}
	|\hat{\sigma}^b_0-\sigma^b | 
	<
	2 \sigma^b \sqrt{ K_0 \frac{\log \log \rho n}{\rho n}} + 9 \sigma^b \frac{\log \log n}{\rho n} 
	\leq
	4  \sigma^b \sqrt{ K_0 \frac{\log \log n}{\rho n}} \quad \mbox{a.s.}
	\end{eqnarray}
	Similarly, we can show there exists $K_1>0$ such that for $n$ large enough,
	\begin{equation}
	\label{eq:sig_1_bound}
	|\hat{\sigma}^b_1-\sigma^b |<4  \sigma^b \sqrt{ K_1 \frac{\log \log n}{(1-\rho) n}} \quad \mbox{a.s.}
	\end{equation}
	Now, observe that
	\begin{eqnarray}
	\label{eq:var_new_form}
	\frac{\rho \hat{\sigma}^b_0 +(1-\rho) \hat{\sigma}^b_1}{(\hat{\sigma}^b_0)^{\rho} (\hat{\sigma}^b_1)^{1-\rho} } =
	\rho \bigg( \frac{\hat{\sigma}^b_0}{\hat{\sigma}^b_1} \bigg)^{1-\rho}+
	(1-\rho) \bigg( \frac{\hat{\sigma}^b_0}{\hat{\sigma}^b_1} \bigg)^{-\rho}.
	\end{eqnarray}
	Using \eqref{eq:sig_0_bound} and \eqref{eq:sig_1_bound}, we can show that for $n$ large enough,
	\begin{align}
	\label{eq:varphi_bound}
	\bigg|  \frac{\hat{\sigma}^b_0}{\hat{\sigma}^b_1} -1 \bigg|  
	&= \bigg| \frac{\hat{\sigma}^b_0-\hat{\sigma}^b_1}{\hat{\sigma}^b_1} \bigg| \quad \mbox{a.s.}
	\nonumber \\ 
	&\leq \frac{2}{\sigma^b} \left| \hat{\sigma}^b_0-\hat{\sigma}^b_1 \right| \quad \mbox{a.s.} \nonumber \\ 
	&\leq
	\frac{2}{\sigma^b} \left(  \left| \hat{\sigma}^b_0-\sigma^b \right| + \left| \hat{\sigma}^b_1-\sigma^b \right| \right) \quad \mbox{a.s.}
	\nonumber \\ 
	&\leq  K \bigg( \frac{1}{\sqrt{\rho}} + \frac{1}{\sqrt{1-\rho}} \bigg) \sqrt{\frac{\log \log n}{n}} \quad \mbox{a.s.},
	\end{align}
	where $K=8 \max \{ \sqrt{K_0},\  \sqrt{K_1}  \}$.
	By Lemma~S3 in Supplementary Material A, there exists $r>0$ such that for all $t \in (0,1)$ and $x \in (1-r, 1+r)$,
	\begin{equation}
	\label{eq:hokm_mc}
	t x^{1-t}+(1-t)x^{-t} \leq 1+t(1-t)(x-1)^2.
	\end{equation}
	Using \eqref{eq:var_new_form}, \eqref{eq:varphi_bound}, and \eqref{eq:hokm_mc}, we see that for $n$ large enough,
	\begin{align}
	\label{eq:var_frac_bound}
	\frac{\rho \hat{\sigma}^b_0 +(1-\rho) \hat{\sigma}^b_1}{(\hat{\sigma}^b_0)^\rho (\hat{\sigma}^b_1)^{1-\rho} } 
	&\leq 1+ K^2 \rho (1-\rho)
	\bigg( \frac{1}{\sqrt{\rho}} + \frac{1}{\sqrt{1-\rho}} \bigg)^2 \frac{\log \log n}{n}  \quad \mbox{a.s.} \nonumber \\ 
	&\leq
	1+2K^2 \frac{\log \log n}{n}  \quad \mbox{a.s.},
	\end{align}
	where in the last inequality we have used \eqref{eq:rho_bound}. Combining \eqref{eq:bad_var}, \eqref{eq:mean_term_bound_2}, and \eqref{eq:var_frac_bound}  we see that for $n$ large enough,
	\begin{align}
	\label{eq:conv_bad_in}
	n^{-q} \bigg( \frac {\hat{\sigma}^b}
	{(\hat{\sigma}^b_0)^{\rho} (\hat{\sigma}^b_1)^{1-\rho}} \bigg) ^{0.5n}
	&<
	n^{-q} \bigg(  1 + \left( 16+ 2K^2\right)  \frac{\log \log n}{n}  \bigg)^{0.5n} 
	<
	n^{-q} ( \log n) ^{K^2+8} \quad \mbox{a.s.},
	\end{align}
	where in the last inequality we have used the fact that for all $x,t>0$, $(1+t/x)^x<e^t$. Since the limit of the right-hand side is 0 whenever $q > 0$, so is that of the left-hand side. Combining \eqref{eq:conv_bad_out} and \eqref{eq:conv_bad_in} we see that \eqref{eq:bad_case} holds almost surely.
\end{proof}

\section{Performance and Consistency on Synthetic Data}

\label{sec:simulations}

Here we implement OBF and several other feature selection methods on synthetically generated microarray data.  An application on real colon cancer microarray data is provided in Sections S2 and S3 of Supplementary Material A. Although OBF assumes all features are independent with Gaussian class-conditional distributions, the data generation model employed violates these assumptions by generating correlated and non-Gaussian features.  
Remarkably, OBF is still theoretically consistent by Theorems~\ref{sec:mar_cor} and~\ref{thm:convergence}.  Since the main contributions of this paper are theoretical, and numerous extensive simulation studies have already shown that OBF has competitive and robust performance~\citep{ref10,dalton2018heuristic,foroughipour2018optimal}, our primary objective in this section is to simply observe whether OBF is indeed consistent, i.e. whether it eventually selects the correct feature set as sample size grows.  
Our secondary objective is to provide new examples showing that OBF enjoys competitive performance, running time, and memory consumption compared with popular Bayesian and non-Bayesian feature selection algorithms, including several methods that OBF has not been compared with before.  

The data is generated using a variant of the ``synergetic'' model originally proposed in~\cite{hua2009performance}.  For a fixed sample size, $n$, in each iteration we assign an equal number of points to class $0$ and $1$ ($n$ is always even).  We generate $|F| = 20,000$ features, including a random assignment of $20$ \emph{global markers}, $80$ \emph{heterogeneous markers}, $11,900$ \emph{low-variance non-markers} and $8,000$ \emph{high-variance non-markers}. Markers have distinct class conditional distributions, non-markers have identical distributions in both classes, and heterogeneous markers and high-variance non-markers account for unknown subclasses in the data.  Global markers, heterogeneous markers, and low-variance non-markers are randomly partitioned into blocks of size $k=5$.  All features within a block are correlated, while all blocks of markers, all blocks of low-variance non-markers, and all high-variance non-markers are independent from each other.  All features are also randomly assigned to one of four groups, $i = 0, 1, 2, 3$, such that each group contains one block of global markers, four blocks of heterogeneous markers, $595$ blocks of low-variance non-markers and $2,000$ high-variance non-markers.  

We now focus on how data is generated in group $i$.  The single block of global markers is jointly Gaussian in class $y=0,1$ with mean $\mu_y$ and covariance matrix $\Sigma_{y,i} = \sigma_{y,i} \Sigma$, where $\mu_0 = [0, \ldots, 0]$, $\mu_1 = [1, 1/2, \ldots, 1/k]$, diagonal elements of $\Sigma$ are $1$, and off-diagonal elements are $\rho=0.8$.  To generate heterogeneous markers, points in class $1$ are further partitioned into $c=2$ roughly equal size subclasses (when $n_1 = n/2$ is odd, subclass $0$ is assigned one more point than subclass $1$).  For two blocks of heterogeneous markers, points in class $0$ or subclass $0$ of class $1$ are drawn from $\mathcal{N}(\mu_0, \Sigma_{0,i})$ and points in subclass $1$ of class $1$ are drawn from $\mathcal{N}(\mu_1, \Sigma_{1,i})$.  For the remaining two blocks, points in class $0$ or subclass $1$ of class $1$ are drawn from $\mathcal{N}(\mu_0, \Sigma_{0,i})$ and points in subclass $0$ of class $1$ are drawn from $\mathcal{N}(\mu_1, \Sigma_{1,i})$.  Each block of low-variance non-markers is jointly Gaussian with mean $\mu_0$ and covariance matrix $\Sigma_{0,i}$ in both classes.  High-variance non-markers are independent and drawn from the mixture of Gaussians $p\mathcal{N}(0,\sigma_{0,i})+(1-p)\mathcal{N}(1,\sigma_{1,i})$, where $p$ is independently drawn from a uniform distribution over $(0,1)$ for each feature. We set $\sigma_{0,0}=\sigma_{1,0}=0.16$, $\sigma_{0,1}=\sigma_{1,1}=0.49$, $\sigma_{0,2}=0.09$, $\sigma_{1,2}=0.25$, $\sigma_{0,3}=0.49$ and $\sigma_{1,3}=0.64$. These values were originally suggested in~\cite{hua2009performance}.  Also note that in~\cite{hua2009performance}, there is only one group, and low-variance non-markers are all independent rather than being assigned to blocks.

We implement four variants of Gaussian OBF: MNC-OBF-PP, CMNC-OBF-PP, MNC-OBF-JP and CMNC-OBF-JP.  PP refers to a proper prior with $s^f_0=s^f_1=s^f=0.5$, $\kappa^f_0=\kappa^f_1=\kappa^f=3$, $m^f_0=m^f=0$, $m^f_1=0.2$ and $\nu^f_0=\nu^f_1=\nu^f=0.1$ for all $f \in F$.  These $\kappa$'s are the smallest integer values where $E(\sigma^f_0)$, $E(\sigma^f_1)$ and $E(\sigma^f)$ exist. JP is based on Jeffreys non-informative prior, and sets $L^f=0.1$, $s^f_0=s^f_1=s^f=0$, $\kappa^f_0=\kappa^f_1=\kappa^f=0$ and $\nu^f_0=\nu^f_1=\nu^f=0$ for all $f$. When $\nu$'s are $0$, $m$'s need not be specified. We set $\pi(f)=0.005$ for all $f$ under PP and JP.
Under MNC, we select all features $f$ such that $\pi^*(f) = h(f)/(1+h(f)) > 0.5$, where $h(f)$ is given in~\eqref{eq:definition_of_h_Gaussian}.  For MNC, the choice of $\pi(f)$ (and $L^f$ under improper priors) affects the average number of features selected; larger $\pi(f)$ and $L^f$ produce larger feature sets.  Under CMNC we select the $D=100$ features maximizing the right-hand side of~\eqref{eq:IGM_improper}.  CMNC-OBF-JP reduces to minimizing 
$(\hat{\sigma}^f_0 )^{0.5n_0} (\hat{\sigma}^f_1 )^{0.5n_1} / (\hat{\sigma}^f )^{0.5n}$, which is essentially the~\cite{NP_OBF_JP_freq} statistic.  
For CMNC, as long as $\pi(f)$ and $L^f$ are constant for all $f$, their values do not affect the rank of features and thus need not be specified.  

In addition to OBF, we implement:
Welch's t-test (t-test), 
a moderated t-test from the \texttt{limma} package in R~\citep{smyth2004linear} (Moderated t-test), 
the Bhattacharyya distance between Gaussian distributions with sample means and variances computed from each class (BD), 
the mutual information between features and class labels computed from a non-parametric entropy estimator based on sample spacings of order $m = 1$~\citep{beirlant1997nonparametric} (MI), 
and a bolstered error estimate~\citep{braga2004bolstered} under nearest mean classification (NMC).
In each case, we output the $D=100$ top ranked features.  
Note that these methods are all univariate filters.  

We also implement $84$ regularized regression methods, using three link functions (linear regression, a GLM with logit link, and a GLM with probit link), two penalty families (LASSO and elastic net), and $14$ regularization parameters (using MATLAB's \texttt{lassoglm} function we set $\lambda = 0.1, 0.2, 0.5$, $\lambda = \gamma/\sqrt{n}$ for $\gamma=0.1,0.2,0.5,1,2,5,10$, $\lambda = (\log n)^{\gamma}/n$ for $\gamma=0.5,1,1.5,2$, and $\alpha = 0.5$).  See~\cite{zou2006adaptive} for properties of LASSO under these families of regularization parameters. For each regression method, we output the set of features used in the regression model.  

Finally, we implement three types of Bayesian variable selection methods: 
the univariate filter method SKBS~\citep{lock2015shared}, 
a regression method using a slab-and-spike prior and probit link~\citep{lee_gene_2003} (BPM), 
and a regression method by~\cite{2016arXiv161106649M} (BayesReg). 
Due to the high computation cost of these methods, we run each on the top $300$ features as ranked by BD, rather than on the full set of $20,000$ features. 
We implement SKBS with $K = 2$ to $K=7$ Gaussian mixture kernels.  We observed best performance with $K=2$ and report on only this case.  As in~\cite{lock2015shared}, we find the marginal posterior probability of each feature having distributional differences using a Gibbs sampler with a burn-in period of $1,000$ steps and a sampling period of $5,000$ steps.  We report the $D=100$ features having largest marginal posteriors with ties broken by BD (CMNC-SKBS), and the set of all features with marginal posteriors greater than $T=0.9$ (MR-SKBS).  We also implemented $T=0.5$ (the threshold of MNC) and $T=0.75$, but observed best performance with $T=0.9$.
We implement BPM using default settings in the published code, except we initialize the MCMC chain with the top $D=100$ features ranked by BD, forgo the burn-in period, and directly generate $5,000$ samples.  Similar to CMNC-SKBS, we then report the $D=100$ features having largest marginal posteriors with ties broken by BD.   
We implement four variations of BayesReg using default settings in the published MATLAB code. Each variant corresponds to one combination of prior ($L_1$ or horseshoe) and link function (linear or logit).  BayesReg outputs a t-statistic, and for each variant of BayesReg we report the $D=100$ features with largest absolute t-statistic.  

This procedure is iterated $600$ times for each $n$, where $n$ increases from $50$ to $1,000$ in steps of $50$.  For each algorithm, reported features are labeled markers and unreported features are labeled non-markers.  
Figure~\ref{fig:hua}(a) shows the average number of correctly labeled features over iterations with respect to $n$.  For each $n$, Regularized-best presents the best performance observed among all $84$ regularized regression methods, and BayesReg-best presents the best performance among all four BayesReg methods.  

\begin{figure}[t]
	\centering
	\begin{subfigure}[b]{0.45\textwidth}
		\includegraphics[width=\textwidth]{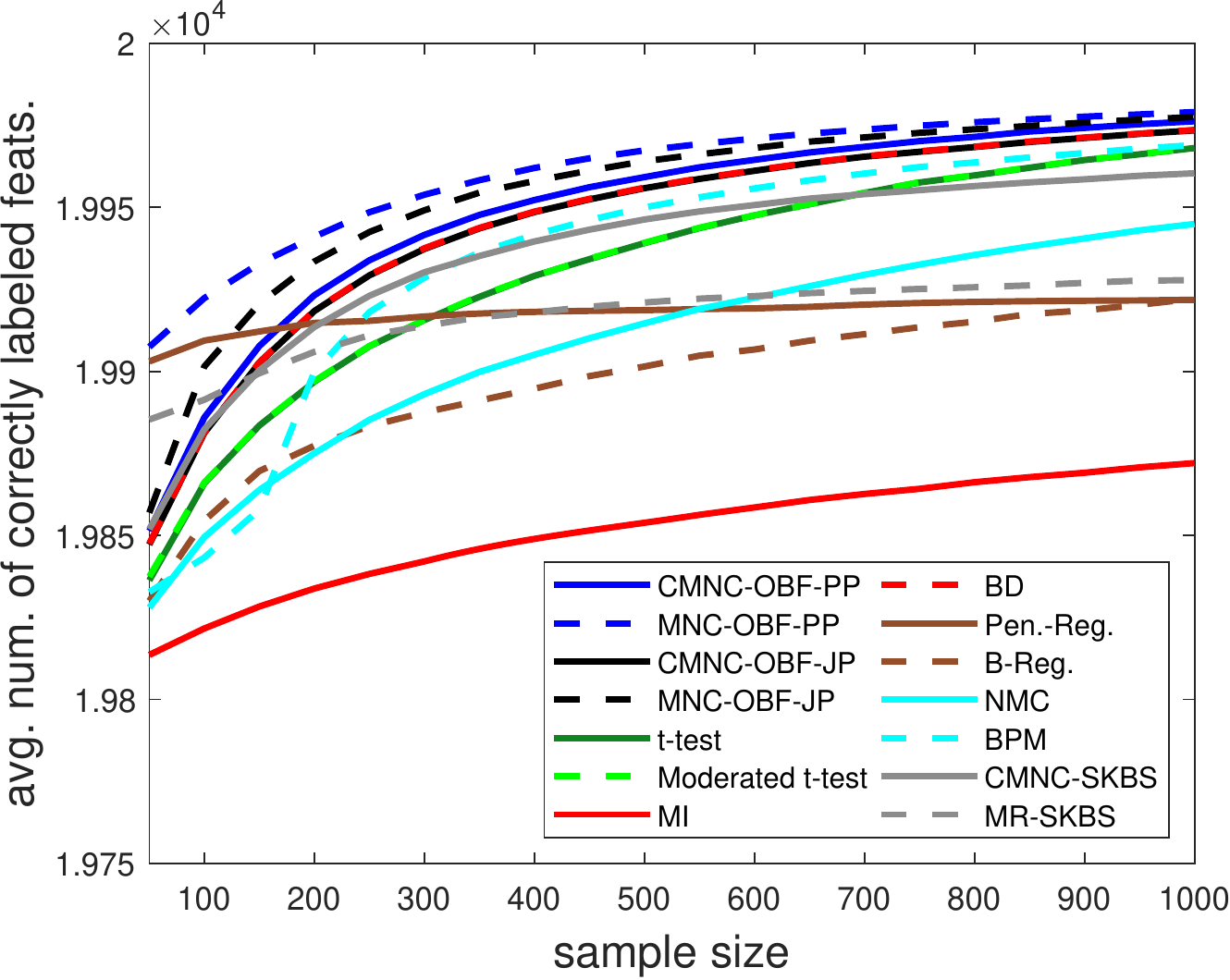}
		\caption{}
	\end{subfigure}
	\begin{subfigure}[b]{0.45\textwidth}
		\includegraphics[width=\textwidth]{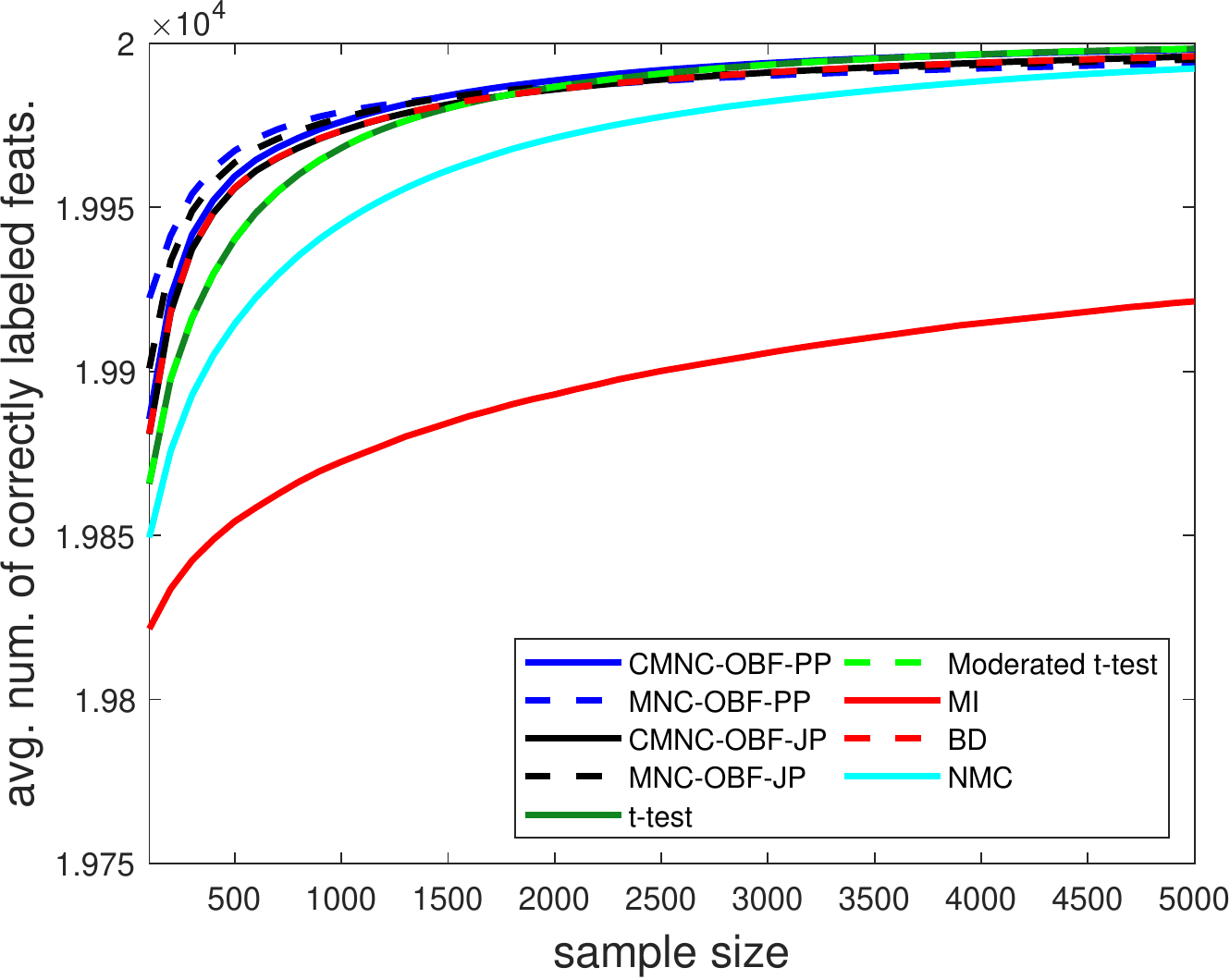}
		\caption{}
	\end{subfigure}
	\caption{Average number of correctly labeled features versus sample size for a synthetic microarray model. (a) All feature selection algorithms for $n$ up to $1,000$; (b) Univariate filters for $n$ up to $5,000$.}
	\label{fig:hua}
\end{figure}

In general, the best performing algorithm is MNC-OBF-PP, which is followed by MNC-OBF-JP, CMNC-OBF-PP, then CMNC-OBF-JP and BD.  OBF and BD perform well because they can detect differences between both means and variances~\citep{foroughipour2018optimal}.  Observe that PP outperforms JP.  In general, an informed prior like PP can have better performance than a non-informative prior like JP when assumptions are accurate, but may be less robust when assumptions are inaccurate.  Also observe that MNC outperforms CMNC.  In general, MNC outperforms CMNC when the sample size is small, and CMNC slightly outperforms MNC when the sample size is large. It may seem counterintuitive for MNC to outperform CMNC, since CMNC is directly informed with the true number of markers to select (via $D$) and MNC is not.  However, MNC is given some information about the number of markers through $\pi(f)$---recall that the expected number of good features given $\pi(f)$ can be found in~\eqref{eq:car_pri}.  In addition, MNC outputs a variable number of features, and under small samples it can be beneficial to output a smaller feature set to avoid selecting features that one is uncertain about.  Also note that CMNC-OBF-JP and BD make similar assumptions, and typically have very similar performance, as seen here.  

Regularized-best and MR-SKBS appear to perform very well under small samples; however, these methods are the only methods besides MNC-OBF-PP and MNC-OBF-JP that output a variable number of features, and they perform very close to the trivial algorithm that outputs no features (which always labels $19,900$ features correctly).  
CMNC-SKBS also performs fairly well under small samples, but drops below BPM at around $n = 300$ and below Moderated t-test at around $n = 650$.  This may be due to insufficient sampling iterations of the Gibbs sampler, or an issue with selecting the number of kernels. 
Since SKBS models mixtures of Gaussians, it can detect differences between means and variances like OBF and BD, and potentially differences between higher order moments, but performance may be sensitive to the number of kernels used.  

Under large samples, BD is followed by BPM, t-test, CMNC-SKBS, then NMC.  
BPM appears to perform close to BD under large samples because its MCMC chain is initialized with BD.  
Unlike OBF and BD, t-test and NMC struggle to detect features with similar means but different variances, which usually results in some loss in performance relative to BD, with NMC performing worse than t-test~\citep{foroughipour2018optimal}.  
BayesReg-best has comparable performance to Regularized-best and MR-SKBS, while MI has the poorest performance across all sample sizes.  
Although MI does not perform well here, as a non-parametric method it can detect any distributional differences, and it has been observed that MI can shine under large differences in skewness~\citep{foroughipour2018optimal}.  

All univariate filters (OBF, t-test, Moderated t-test, BD, MI and NMC) do not account for correlations between features, while all regression based methods we implemented (Regularized-best, BPM and BayesReg) do account for correlations.  Regression-based methods do not perform particularly well, except Regularized-best under small samples (where it tends to output very few features) and BPM under large samples (where performance tracks BD because the MCMC chain is initialized with BD).  As discussed in Section~\ref{sec:1_introduction}, classification and regression based methods tend to miss weak features in the presence of strong features, and miss strong features that are correlated to other stronger features, because these features are not very useful in improving the predictive capacity of the model.  See Section S4 of Supplementary Material A for more discussion on this.  

Table~\ref{tab:compcost1} lists the average running time and maximum memory requirement of several methods for $n=200$ over $10$ iterations.  MNC-OBF-PP, CMNC-OBF-PP, MNC-OBF-JP and CMNC-ONF-JP have similar computation cost and are reported in the table as ``OBF.''  t-test and Moderated t-test have comparable computation cost and are averaged together in the table and reported as ``t-test.''  ``Regularized-best'' reports the average computation cost for all $84$ regularized regression models.  We implement SKBS with $K = 2$ kernels, SKBS with $K = 7$ kernels, BPM, and the four earlier variants of BayesReg after filtering out all but the top $300$ features with BD, and again after filtering out all but the top $5,000$ features.  ``BayesReg-best'' reports the average computation cost of all four variants of BayesReg. OBF is not only the best performing, but also stands among the fastest methods with low memory requirements.  SKBS, BPM and BayesReg all have running times that are orders of magnitude higher than that of OBF and require several times the amount of memory, particularly when run on a larger number of features. Our code is vectorized, which tends to reduce running time at the cost of higher memory consumption.  

\begin{table}[t]
	\caption{Computation Cost of Feature Selection Algorithms}
	\centering
	\resizebox{\linewidth}{!}{
		\begin{tabular}{c|cccccccc}
			\hline
			Method	&	OBF		&	t-test	&	NMC		&	SKBS(300, $K=2$)	&	SKBS(300, $K=7$)	&	BPM(300)	&	BayesReg-best(300)	\\
			Running time	&	$1$		&	$1$		&	$7$		&	$400$			&	$800$			&	$2000$		&	$300$	\\	
			Memory	&	$<5$MB	&	$<5$MB	&	$<7$MB	&	$20$MB			&	$33$MB			&	$30$MB		&	$25$MB	\\
			\hline	
			Method	&	BD		&	MI			&	Regularized-best	&	SKBS(5000, $K=2$)	&	SKBS(5000, $K=7$)	&	BPM(5000)			&	BayesReg-best(5000)\\	
			Running time	&	$0.9$	&	$5$			&	$30$		&	$>10000$				&	$>10000$			&	$>20000$	&	$> 10000$	\\	
			Memory	&	$<5$MB	&	$<50$MB		&	$<5$MB		&	$300$MB				&	$>500$MB			&	$350$MB				&	$75$MB	\\		
			\hline
			\multicolumn{8}{l}{*OBF running time is taken as the unit of time.}\\
		\end{tabular}
	}
	\label{tab:compcost1}
\end{table}

We conclude this section with a simulation similar to that of Fig.~\ref{fig:hua}(a), except we do not implement computationally intensive methods and we let sample size increase from $100$ to $5,000$ in steps of $100$. Figure \ref{fig:hua}(b) plots the average number of correctly labeled features with respect to sample size. The curves for BD, t-test, Moderated t-test, and all methods based on OBF appear to converge to $20,000$, which suggests that these methods are consistent under the current data model.  It is also interesting that t-test becomes more competitive for very large sample sizes. 

\section{Conclusion}
\label{sec:conclusion}

OBF should not be used in applications where the objective is dimensionality reduction to design a simpler model or avoid overfitting.  Rather, it is designed for applications where \emph{all} features that exhibit distributional differences between the classes should be ranked and reported. 
That being said, as a filter method, OBF cannot identify a feature that is itself indistinguishable between the classes, while being highly correlated with other features that do have distributional differences.  Such features are of interest in biomarker discovery because: (1) they might be paired with other biomarkers to develop better tests for the biological condition of interest, and (2) strong correlations between genes or gene products suggest possible links in the underlying biological mechanisms, and understanding these links is an important part of the discovery process.  Therefore, a major thrust of our future work is in developing models and methods that can take advantage of correlations.  A few suboptimal methods have been proposed in prior works~\citep{ref7, ref9, ref10, dalton2018heuristic}, however, more work is needed in identifying conditions under which these algorithms are consistent, and in understanding performance and robustness properties of these algorithms.

Finally, note that the OBF framework makes it possible to conduct a Bayesian error analysis for feature selection, much like Bayesian error estimation in classification \citep{Dalton:10I,Dalton:10II}.  For instance, one may find the probability $p(G|S) = P(\bar{G} = G |S)$ in~\eqref{eq:pos_CIM_2_prod} or the expectation $E(\ell(G, \bar{G})|S)$ in~\eqref{eq:RMNC_2} for an arbitrary feature set $G$, or find the ROC curve defined in~\eqref{eq:ROC} for an arbitrary feature selection rule.  We plan to study Bayesian error analysis under the OBF framework in future work, and to develop and study methods of error analysis that also take into account correlations.  




\begin{acknowledgement}
	This work is supported by the National Science Foundation (CCF-1422631 and CCF-1453563).
\end{acknowledgement}

\newpage
\setcounter{page}{1}
\setcounter{section}{0}
\setcounter{figure}{0}
\setcounter{table}{0}


\begin{frontmatter}
\title{Theory of Optimal Bayesian Feature Filtering: \\ Supplementary Material A}

\runtitle{Supplementary Material A}

\begin{aug}

\author{\fnms{Ali} \snm{Foroughi pour}\thanksref{addr1}\ead[label=e1]{foroughipour.1@osu.edu}}
\and
\author{\fnms{Lori A.} \snm{Dalton}\thanksref{addr1}\ead[label=e2]{dalton@ece.osu.edu}}

\runauthor{A. Foroughi pour and L. A. Dalton}

\address[addr1]{Department	of Electrical and Computer Engineering, The Ohio State University, Columbus, OH, 43210
	\printead{e1}
	\printead{e2}
}

\end{aug}



\end{frontmatter}


\thispagestyle{empty}

\renewcommand{\thesection}{S\arabic{section}}
\renewcommand{\thetable}{S\arabic{table}}
\renewcommand{\thefigure}{S\arabic{figure}}
\renewcommand{\theequation}{S\arabic{section}.\arabic{equation}}
\renewcommand{\themylem}{S\arabic{mylem}}

\section{Proof of Lemmas}

\label{sec:lemmas}

Here we provide the three lemmas used by Theorem $3$ of the main manuscript.

\begin{mylem}
	Suppose $G,H \subseteq F$, $p(G),p(H)>0$, $p(\theta|G)$ is semi-proper, and $S_{\infty}$ is a balanced sample. Then there exists $c>0$ and $q>0$ such that
	\begin{equation}
	\label{eq:p_trend}
	\frac{z(G, S_n)}{z(H, S_n)} \sim c n^{q(|H|-|G|)}
	\quad \text{as $n \to \infty$}.
	\end{equation}
	\label{lemma:p_trend}
\end{mylem}
\vspace{-10mm}

\begin{proof}
	Throughout this proof, note that $\rho$ is a function of $n$. Let $f \in F$. Observe that as $n \to \infty$,
	\begin{eqnarray}
	\label{eq:nu_up}
	(2\pi)^{0.5} 0.5^{0.5(\kappa^f - \kappa^f_0 - \kappa^f_1)}
	\bigg( \frac{\nu^{f*}}{\nu^{f*}_0 \nu^{f*}_1} \bigg)^{0.5} \sim c_1 \left( \rho (1-\rho) \right)^{-0.5}  n^{-0.5},
	\end{eqnarray}
	for some $c_1>0$. Now, using Stirling's formula we see that
	\begin{align}
	\label{eq:stir}
	&\frac{\Gamma(0.5\kappa^{f*}_0) \Gamma(0.5\kappa^{f*}_1)}{\Gamma(0.5\kappa^{f*})}
	\sim \bigg(\frac{2 \pi (0.5\kappa^{f*}_0-1) (0.5\kappa^{f*}_1-1)}{0.5\kappa^{f*}-1} \bigg)^{0.5}
	e^{1+0.5(\kappa^f-\kappa^f_0-\kappa^f_1)} \nonumber \\
	&\qquad\qquad \times
	\big( 0.5\kappa^{f*}-1\big)^{-(0.5\kappa^{f*}-1)} 
	\big( 0.5\kappa^{f*}_0-1\big)^{0.5\kappa^{f*}_0-1}
	\big( 0.5\kappa^{f*}_1-1\big)^{0.5\kappa^{f*}_1-1}
	\end{align}
	as $n \to \infty$.  Note that,
	\begin{equation}
	\label{eq:sqrt_up}
	\bigg(\frac{2 \pi (0.5\kappa^{f*}_0-1) (0.5\kappa^{f*}_1-1)}{0.5\kappa^{f*}-1}\bigg)^{0.5}
	e^{1+0.5(\kappa^f-\kappa^f_0-\kappa^f_1)} \sim c_2
	\left( \rho (1-\rho) \right)^{0.5} 
	n^{0.5}
	\end{equation}
	as $n \to \infty$ for some $c_2>0$. In addition,
	\begin{align}
	\bigg( \frac{0.5\kappa^{f*}_0-1}{\rho n -1} \bigg)^{0.5\kappa^{f*}_0} 
	&=
	\bigg( \frac{0.5(\kappa^f_0+\rho n)-1}{\rho n -1} \bigg)^{0.5\kappa^{f*}_0} 
	\sim c_3 0.5^{0.5 \rho n} , \nonumber \\
	\bigg( \frac{0.5\kappa^{f*}_1-1}{(1-\rho) n -1} \bigg)^{0.5\kappa^{f*}_1} 
	&=
	\bigg( \frac{0.5(\kappa^f_1+(1-\rho) n)-1}{(1-\rho) n -1} \bigg)^{0.5\kappa^{f*}_1} 
	\sim c_4 0.5^{0.5(1-\rho) n}, \nonumber \\
	\bigg( \frac{0.5\kappa^{f*}-1}{n -1} \bigg)^{0.5\kappa^{f*}} 
	&=
	\bigg( \frac{0.5(\kappa^f+n)-1}{ n -1} \bigg)^{0.5\kappa^{f*}} 
	\sim c_5 0.5^{0.5n},
	\end{align}
	as $n \to \infty$ for some $c_3,c_4,c_5>0$. Hence,
	\begin{eqnarray}
	\label{eq:frac_up}
	\bigg( \frac{0.5\kappa^{f*}_0-1}{\rho n -1} \bigg)^{0.5\kappa^{f*}_0} 
	\bigg( \frac{0.5\kappa^{f*}_1-1}{(1-\rho) n -1} \bigg)^{0.5\kappa^{f*}_1}
	\bigg(\frac{0.5\kappa^{f*}-1}{n -1} \bigg)^{-0.5\kappa^{f*}} \sim c_6
	\end{eqnarray}
	as $n \to \infty$ for some $c_6>0$.  Furthermore,
	\begin{eqnarray}
	\label{eq:one_term}
	\frac{0.5 \kappa^{f*}-1}{(0.5 \kappa^{f*}_0-1)(0.5 \kappa^{f*}_1-1)} \sim c_7 n^{-1}
	\end{eqnarray}
	as $n \to \infty$ for some $c_7>0$. Since $p(\theta|G)$ is semi-proper and~\eqref{eq:nu_up}, \eqref{eq:stir}, \eqref{eq:sqrt_up}, \eqref{eq:frac_up}, and \eqref{eq:one_term} hold, we see that $l(f, S_n) \sim C n^{p-1}$ for some $C>0$ and $p<1$.  Thus, \eqref{eq:p_trend} holds for some $c>0$ and $q = 1-p>0$.
\end{proof}

\begin{mylem}
	Let $g \in F$ such that either $\mu_0^g \neq \mu_1^g$ or $\sigma_0^g \neq \sigma_1^g$.  In addition, let $S_{\infty}$ be a fixed and balanced sample in which 
	$\hat{\mu}_y^g$ converges to $\mu^g$ and
	$\hat{\sigma}_y^g$ converges to $\sigma_y^g$.
	Then, there exists $\epsilon > 0$ and $c > 0$ such that for $n$ large enough,
	\begin{equation}
	\frac{(c^g_0)^{\kappa^{g*}_0} (c^g_1)^{\kappa^{g*}_1}}{(c^g)^{\kappa^{g*}}} < c (1-\epsilon)^{n}.
	\end{equation}
	\label{lemma:good_variance}
\end{mylem}
\vspace{-10mm}

\begin{proof}
	It suffices to show there exists $\epsilon, c>0$ such that
	\begin{equation}
	\frac{(c^g_0)^\rho (c^g_1)^{1-\rho}}{c^g} < 1-\epsilon
	\label{eq:good_feat_epsilon}
	\end{equation}
	and
	\begin{equation}
	\frac{(c^g_0)^{\kappa^{g}_0} (c^g_1)^{\kappa^{g}_1}}{(c^g)^{\kappa^{g}}} < c
	\label{eq:good_feat_c}
	\end{equation}
	for all $n$ large enough.
	Observe that
	\begin{equation}
	c^g_y= \hat{\sigma}^g_y+\frac{s^g_y}{n_y-1}+\frac{\nu^g_y n_y}{(\nu^g_y + n_y)(n_y-1)} (\hat{\mu}^g_y-m^g_y)^2.
	\end{equation}
	The first term converges to $\sigma^g_y$, and the second and third terms converge to 0, thus, $c^g_y \to \sigma^g_y$.   Similarly,
	\begin{equation}
	c^g= \hat{\sigma}^g+\frac{s^g}{n-1}+\frac{\nu^g n}{(\nu^g + n)(n-1)} (\hat{\mu}^g-m^g)^2.
	\end{equation}
	The second term converges to 0, and, since $\hat{\mu}^g = \rho \hat{\mu}^g_0 + (1-\rho) \hat{\mu}^g_1$, $\hat{\mu}^g$ is bounded and the third term converges to 0.  Also, the following property holds:
	\begin{align}
	\hat{\sigma}^g
	&= \rho  \hat{\sigma}_0^g +(1-\rho)  \hat{\sigma}_1^g
	- \frac{(1-\rho)  \hat{\sigma}_0^g}{n-1}
	- \frac{\rho  \hat{\sigma}_1^g}{n-1} 
	+ \frac{\rho(1-\rho) n}{n-1} (\hat{\mu}_0^g - \hat{\mu}_1^g)^2.
	\label{eq:bad_var_lemma}
	\end{align}
	Since $\hat{\sigma}_y^g$ converges, it is bounded. Hence, $\lim (1-\rho)\hat{\sigma}_0^g/(n-1)=0$, and $\lim \rho  \hat{\sigma}_1^g/(n-1)=0$.  Combining all of this and applying properties of $\limsup$ and $\liminf$, we have,
	\begin{align}
	&\limsup \frac{(c_0^g)^{\rho} (c_1^g)^{1-\rho} } {c^g} 
	= \limsup \frac{(\sigma_0^g)^{\rho} (\sigma_1^g)^{1-\rho} }
	{\rho \sigma^g_0 +(1-\rho) \sigma^g_1 +\rho(1-\rho)(\mu_0^g - \mu_1^g)^2}.
	\end{align}
	We first show~\eqref{eq:good_feat_epsilon} holds if $\sigma_0^g \neq \sigma_1^g$.  Note that
	\begin{align}
	\limsup \frac{(c_0^g)^{\rho} (c_1^g)^{1-\rho} } {c^g}
	&\leq \limsup \frac{(\sigma_0^g)^{\rho} (\sigma_1^g)^{1-\rho} }
	{\rho  \sigma_0^g +(1-\rho)  \sigma_1^g} 
	\leq \max_{r \in [\rho_1,  \rho_2]}
	\frac{(\sigma_0^g)^{r} (\sigma_1^g)^{1-r} }
	{r  \sigma_0^g +(1-r)  \sigma_1^g},
	\label{eq:equal_var_case}
	\end{align}
	where $\rho_1=\liminf \rho$ and $\rho_2=\limsup \rho$. Since we are maximizing a continuous function with respect to $r$, the maximum is obtained for some $r^* \in [\rho_1,  \rho_2]$. Observe for all $0< r<1$, in particular $r^*$, and $x, y>0$, $r x + (1-r) y \geq x^r y^{1-r}$ with equality if and only if $x=y$, which is shown by taking the logarithm of both sides and noting the concavity of logarithm.
	Thus, the right-hand side of~\eqref{eq:equal_var_case} is strictly less than 1.
	Now suppose $\sigma^g \equiv \sigma_0^g = \sigma_1^g$ and $\mu_0^g \neq \mu_1^g$. We have,
	\begin{align*}
	\limsup \frac{(c_0^g)^{\rho} (c_1^g)^{1-\rho} } {c^g}
	&= \frac{\sigma^g} {\sigma^g+ (\mu_0^g - \mu_1^g)^2 \liminf \rho (1-\rho) }.
	\end{align*}
	The right-hand side is again strictly less than 1, therefore, \eqref{eq:good_feat_epsilon} also holds in this case for some $\epsilon > 0$ and $n$ large enough.  Now for~\eqref{eq:good_feat_c}, observe from \eqref{eq:bad_var_lemma},
	\begin{align}
	\limsup \frac {(c^g_0)^{\kappa^{g}_0} (c^g_1)^{\kappa^{g}_1}} {(c^g)^{\kappa^{g}}}
	&\leq \frac{ ( \sigma_0^g)^{\kappa^{g}_0}   (\sigma_1^g)^{\kappa^{g}_1} }
	{(\liminf \rho \sigma^g_0 + (1-\rho) \sigma^g_1)^{\kappa^{g}} } 
	\leq  \frac{ ( \sigma_0^g)^{\kappa^{g}_0}   (\sigma_1^g)^{\kappa^{g}_1} }
	{(\min \{ \sigma^g_0, \ \sigma^g_1  \})^{\kappa^{g}}}
	\nonumber.
	\end{align}
	Thus~\eqref{eq:good_feat_c} holds for some $c > 0$ and $n$ large enough.
\end{proof}

\begin{mylem}
There exists $r \in (0,1)$ such that for all $t \in (0, 1)$ and $x \in (1-r, 1+r)$, 
\begin{equation}
 t x^{1- t} + (1- t) x^{- t} \leq 1 +  t (1- t)(x-1)^2.  
\end{equation}
	\label{lemma:taylor}
\end{mylem}
\vspace{-10mm}

\begin{proof}
Let $t \in (0,1)$ be arbitrary. Let $g_t(x)= 1+t(t-1)(x-1)^2-t x^{1-t}-(1-t)x^{-t}$. Observe that $g_t(1)=0$. Also,
\begin{align}
g_t'(x)
&= 2t(t-1)(x-1)-t(1-t) x^{-t} (1-x^{-1}) \nonumber \\
&= t(t-1) \left(  2(x-1)- x^{-t} (1-x^{-1})\right),
\end{align}
so $g_t'(1) = 0$.  
Further,
\begin{align}
g_t''(x) &= t (1-t) (2+tx^{-t-1}-(1+t)x^{-t-2}) \nonumber \\
&= t(1-t)(2 + x^{-(t+2)}(t(x-1) -1)).
\end{align}
Therefore, $g_t''(1) = t(1-t) > 0$, and $x=1$ is a local minimum of $g_t(x)$. Since $t$ is arbitrary, $g_t(1)=0$, $g_t'(1) = 0$, $g_t''(1) > 0$, and $x=1$ is a local minimum of $g_t(x)$ for all $t \in (0,\ 1)$. Let $f_t(x)=2 + x^{-(t+2)}(t(x-1) -1)$. Observe that $f_t(1)=1$ and $f'_t(1)=2t+2$. Therefore, there exists $r>0$ such that for all $t \in(0,\ 1)$ and $x \in (1-r,\ 1+r)$, $g_t''(x)>0$. 


\end{proof}

\section{Application Using Colon Cancer Microarray Data}

\label{sec:RMD}

Here we use t-test, BD, CMNC-OBF-JP, BPM, SKBS with $K=2$, the Bayesian regression model of \cite{2016arXiv161106649M} with logit link and $L_1$ penalty, denoted by BayesReg(logit,$L_1$), and GLMs with logit link and $L_1$ and elastic net penalties to  a colon cancer microarray dataset~\citep{smith_experimentally_2010,FREEMAN2012562} deposited on Gene Expression Omnibus (GEO)~\citep{edgar_gene_2002} with accession number GSE17538 comprised of subseries GSE17536 and GSE17537.  This dataset contains gene expression levels of $238$ patients in different stages of colon cancer. We assign $28$ stage 1 and adenoma patients to class $0$ and the remaining $210$ patients in stages 2-4 to class $1$. The data has been normalized with Bioconductor's \texttt{affy} package using default settings.
Our objective is to identify features that discriminate between early and late stages of cancer. This dataset uses the GPL570 platform containing $54,765$ probes. Probes not mapping to any genes are removed leaving $42,450$ probes. These probes map to $21,049$ distinct genes or gene families. Here we perform feature selection at the probe level. Afterwards, for each gene we identify the probe ranking highest that maps to the gene. We rank genes based on their associated probe. This downstream analysis will be explained in more detail later. Figure \ref{fig:hist_all_pts} plots the histogram of all patients and all probes used for feature selection.

\begin{figure}
	\centering
		\includegraphics[width=0.6\textwidth]{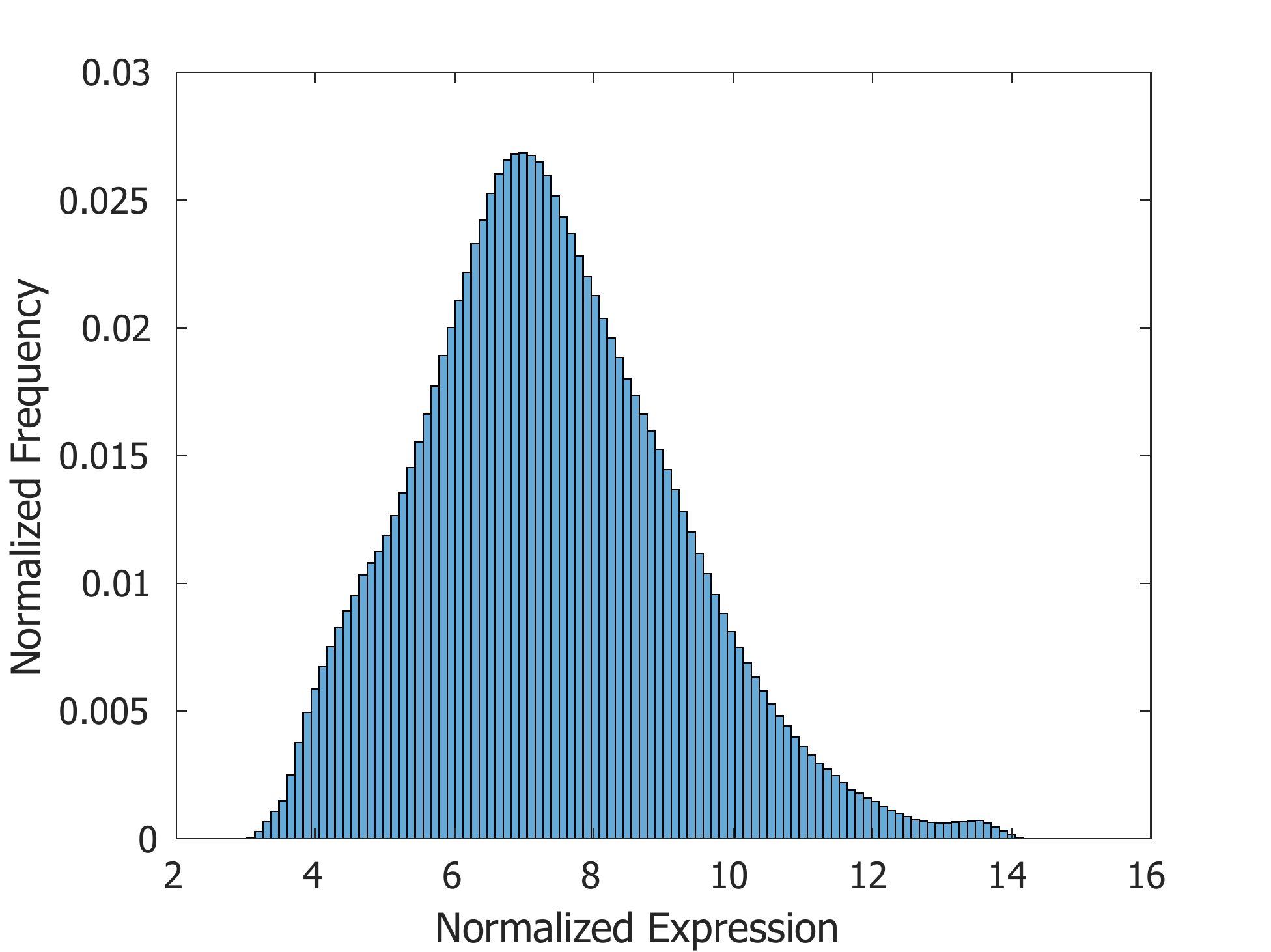}
	\caption{Histogram of normalized expressions of all probes used for feature selection.}\label{fig:hist_all_pts}
\end{figure}

We first study the application of penalized  GLMs with logit link on this dataset. We use MATLAB's built in \texttt{lassoglm} function to implement a logit model with LASSO and elastic net penalties. Elastic net assumes $\alpha=0.5$. We use two methods for model selection under the penalized GLMs: (1) 10 fold cross validation (with 10 Monte Carlo repetitions) to select the penalty resulting in the highest prediction accuracy on the test data, and (2) the stability criterion of~\cite{meinshausen2010stability} to obtain a ``marginal inclusion posterior'' for each feature. The running time of penalized regression under these model selection schemes is discussed in Section~\ref{sec:RTC}.

For cross validation, we consider the penalty family $\lambda=0.005,0.01,\cdots,0.2$. We observed larger values of $\lambda$ output very small feature sets and suffer large prediction error. Although penalty terms smaller than $0.005$ outputted reasonably large feature sets, we again observed large prediction error. Hence, the interval of $[0.005,\  0.2]$ was chosen. 10 fold cross validation selects $\lambda=0.04$ and $\lambda=0.055$ for GLMs with LASSO and elastic net penalties, assigning non-zero regression coefficients to 36 and 83 probes, respectively. All selected probes for LASSO map to distinct genes, and the selected probes for elastic net map to 76 distinct genes. We also use these GLMs with fixed $\lambda=0.01$ and use all of the data for training. This time, 84 and 177 probes mapping to 81 and 165 distinct genes are selected by LASSO and elastic net penalties, respectively.

We now use the stability selection metric of \cite{meinshausen2010stability}. We iterate 100 times, as suggested in \cite{meinshausen2010stability}, and in each iteration we  randomly subsample $90\%$ of the points in each class. Although \cite{meinshausen2010stability} suggests using half of the training data in each iteration to save on computation cost, we used $90\%$ of the training data as suggested in \cite{HE2010215} for small-sample high-dimensional biomarker discovery applications. We then compute $\hat{\Pi}^{\lambda}_k$, the maximum probability of a feature (probe) being selected by a regression model over all penalty terms, as described in \cite{meinshausen2010stability}. We then compare this probability with thresholds $\pi_{thr}=0.1,0.25,0.5,0.75,0.9$. For the LASSO penalty these thresholds select 195, 98, 47, 19, and 7 probes mapping to 186, 93, 46, 19 and 7 different genes, respectively. For the elastic net penalty these thresholds select 471, 258, 133, 75 and 36 probes mapping to 440, 244, 122, 67, and 31 genes, respectively. Equation (9) in \cite{meinshausen2010stability} provides an upper-bound on the expected number of false discoveries for each selected set. Using this equation to bound FDR by $5\%$, LASSO and elastic net both use $\pi_{thr}=0.51$ to select 47 and 133 probes mapping to 46 and 122 genes, respectively. However, with this threshold the upper-bounds on FDR are $7.1 \times 10^{-4}$ and $8.9 \times 10^{-4}$, respectively. The bound of \cite{meinshausen2010stability} is not applicable when $\pi_{thr}<0.5$.

Both cross validation and the stability criterion output a relatively small number of features as markers. While many of these genes are high-profile biomarkers, we observe many important biomarkers are missed by all GLMs. For example, all GLMs (including all $\lambda$'s and $\pi_{thr}$'s considered above) miss TTN, TP53, FBXW7, and CCNE1, which are verified colon cancer biomarkers~\citep{fearon_molecular_2011, network_comprehensive_2012}, except for elastic net with $\pi_{thr}=0.1$ which selects TP$53$. We will later see these verified biomarkers rank high by CMNC-OBF-JP, BD, and BayesReg(logit,$L_1$). TP53 also ranks high by t-test. This is typical of many algorithms based on regression performance; different methods might select different feature sets \citep{saeys_review_2007}, and few features might be declared as biomarkers \citep{sima2006should,sima2008peaking}.

Now we study the application of other selection algorithms on this dataset. We use t-test, BD, CMNC-OBF-JP, BPM, SKBS, and BayesReg(logit,$L_1$) to rank genes. BPM, SKBS, and BayesReg(logit,$L_1$) are computationally intensive. Therefore, we implement a first phase filtration by BD similar to Section 6 of the main manuscript. However, to avoid missing important biomarkers we use top the 5000 BD probes, mapping to 4017 different genes. CMNC-OBF-JP is implemented as in Section 6 of the main manuscript, except we set $\pi(f)=0.01$ for all features $f$ to obtain $\pi^*(f)$. We use $\pi^{S*}(f)$ to denote the SKBS marginal posterior of a feature $f \in F$. After probe ranks are obtained, among probes that have exactly the same associated gene list, we only keep the probe with the highest rank. Afterwards, we rank genes. Thereby, a gene whose associated probe ranks higher in the probe ranking also ranks higher in the gene list.

Table~\ref{tab:ccgr} contains a list of genes in this colon cancer dataset with reported or suggested involvement in colon cancer~\citep{fearon_molecular_2011, network_comprehensive_2012}. While ~\cite{fearon_molecular_2011} and \cite{network_comprehensive_2012} also contain other verified biomarkers, none of the probes in the GPL570 platform (the platform used in this dataset) map to these biomarkers. Hence, they were removed from Table~\ref{tab:ccgr}. For each gene we also list the rank assigned by each method. SKBS, BPM and BayesReg(logit,$L_1$) only rank the top 15 genes in the table; the remaining genes do not pass the first phase of filtration by BD.   Table~\ref{tab:ccrpv} lists the number of these known markers that rank below $N=1000$, $2000$, and $5000$, and the associated p-values, for each method. These p-values are based on an over-representation test using a hyper-geometric distribution. 
%
%
As these tables suggest, CMNC-OBF-JP and BD seem to provide a better feature ranking. Recall that SKBS, BPM, and BayesReg(logit,$L_1$) use a first phase of filtration using top 5000 probes selected by BD, which may affect their performance. SKBS assumes $K=2$, which may be insufficient to model data; indeed, a feature with strong differences might get low $\pi^{S*}(f)$ when the kernels do not properly describe its distribution. While \cite{lock2015shared} use cross validation to select the number of kernels, (a) it is computationally intensive, (b) given the very small sample size in class 0, partitioning the data might affect performance, and (c) cross validation may suggest a large $K$. In the synthetic simulations of Section 6 of the main manuscript, we observed that large values of $K$ may overestimate $\pi^{S*}(f)$, which is not desirable. In Fig. \ref{fig:BAROC}, the x-axis is the number of selected features and the y-axis is the ratio of the number of selected markers over the total number of markers listed in Table \ref{tab:ccgr}. These curves help visualize Tables \ref{tab:ccgr} and \ref{tab:ccrpv}, and illustrate how OBF and BD provide better feature rankings. 

Tables \ref{tab:t20obf} and \ref{tab:t20ttest} list the top 20 genes selected by CMNC-OBF-JP and t-test, respectively, as well as their associated $\pi^*(f)$, $\pi^{S*}(f)$, and t-test p-values. No FDR correction is done for t-test. As these two tables suggest, genes that rank high by t-test typically have large $\pi^*(f)$, but the converse is not true, i.e., top genes from OBF might have large p-values even when no FDR correction is performed. SBKS performs similar to OBF and assigns large posteriors to top genes of OBF and t-test. 
Many of the top OBF genes are also verified colon cancer biomarkers. For example, GAGE genes are over-expressed in a subpopulation of colon cancer patients, and may promote cancer progression \citep{gjerstorff2015oncogenic,Scanlan1}. Furthermore, CPNE4 \citep{shin2009micrornas}, EPHA7 \citep{wang2005downregulation,kim2010dna}, LOC286297 \citep{hu2017epigenomic}, SLC2A2, also known as GLUT2~\citep{lambert2002molecular}, and MAGEA4 \citep{colmage2}  have also been shown or suggested to be involved in colon cancer. Our methodology favors selecting genes that discriminate between early and late stages of disease, thus some genes known to have strong links to colon cancer across all stages, for instance KRAS, may not necessarily rank high.

Histograms of several top OBF genes are provided in Fig. \ref{fig:hist}. They indeed have distributional differences between the two classes. GAGE genes are over-expressed in a small subpopulation of late-stage colon cancer patients (Figs. \ref{fig:hist}(a) and \ref{fig:hist}(b)), which is in line with the literature \citep{gjerstorff2015oncogenic,Scanlan1}. Furthermore, we observe GPM6A is also over-expressed in some late-stage patients by comparing Figs. \ref{fig:hist}(c) and \ref{fig:hist}(d). While the expression of GPM6A is never above $5$ among early-stage cancer patients, among $28.1\%$ of late-stage patients its expression is above 5. GPM6A has been suggested to be a biomarker in colon and lung cancers \citep{camps2009integrative,hasan2015silico}.

SLC14A1 is an interesting biomarker. It has been suggested to be involved in colon cancer \citep{vanErk2005,slc14_colon2}, bladder cancer \citep{SLC14A1_bladder}, and prostate cancer \citep{STAMEY20012171}. However, all of these studies suggest SLC14A1 is under-expressed in cancer. Comparing Figs. \ref{fig:hist}(e) and \ref{fig:hist}(f) we observe SLC14A1 is over-expressed in a sub-population of late-stage colon cancer patients and under-expressed in another subpopulation. The probability that the expression of SLC14A1 is below $4$ is $17.9\%$ in early-stage patients and $22.4\%$ among late-stage patients. In addition, the probability that the expression of SLC14A1 is above $5$ is $0\%$ and $8.6\%$ among early- and late-stage patients, respectively. While we observe SLC14A1 is under-expressed in a subpopulation in accordance with the literature, we found no references justifying over-expression of SLC14A1 in another subpopulation. This is an interesting pattern observed in this dataset motivating further investigation. The net effect of all observed data is a slight under-expression of SLC14A1. The sample mean of early-stage and late-stage patients is $4.15$ and $4.36$, respectively. Hence, methods that only look at sample means, such as t-test, might miss this interesting gene.

Finally, comparing Figs. \ref{fig:hist}(g) and \ref{fig:hist}(h) we observe that MMP8 is over-expressed in a subpopulation of late-stage cancer patients, as suggested in \cite{felipe2013poor}.

As the histograms suggest, these biomarkers are not Gaussian; however, OBF has been successful at identifying probes with strong distributional differences. These results suggest that OBF under JP might enjoy robustness with respect to its modeling assumptions. A detailed discussion of robustness and performance properties of OBF is provided in \cite{TCBB}. Also, \cite{berger} provides a discussion on potential robustness properties of hierarchical Bayesian models with non-informative priors.

\begin{table*}[th!]
	\begin{center}
		\resizebox{0.85\linewidth}{!}{
			\begin{tabular}{|ccccccc|}
				\hline
				Gene & CMNC-OBF-JP & BD & t-test & SKBS & BPM & BayesReg(logit,$L_1$) \\ 
				\hline
				TTN	&	21	&	14	&	6983	&	2383	&	1165	&	772	\\
				SMAD3	&	255	&	228	&	304	&	694	&	1230	&	882	\\
				PTEN	&	440	&	384	&	3934	&	2965	&	1277	&	466	\\
				MLH1	&	804	&	689	&	2607	&	3792	&	2718	&	652	\\
				TP53	&	807	&	711	&	341	&	3588	&	2725	&	195	\\
				ACVR1B	&	956	&	822	&	4845	&	1722	&	2761	&	3358	\\
				CCNE1	&	1274	&	1103	&	13081	&	3392	&	2861	&	2115	\\
				FBXW7	&	1474	&	2221	&	14636	&	390	&	1827	&	2132	\\
				MYB	&	1765	&	1545	&	760	&	3868	&	1602	&	868	\\
				KRAS	&	2166	&	1841	&	4250	&	3911	&	1712	&	3611	\\
				CDC27	&	2345	&	2008	&	1908	&	386	&	808	&	3624	\\
				PIK3CA	&	2432	&	2095	&	6529	&	1167	&	1784	&	1801	\\
				EGFR	&	2716	&	2379	&	11094	&	3446	&	3343	&	3811	\\
				SOX9	&	3312	&	2987	&	1735	&	3912	&	3578	&	2799	\\
				EDNRB	&	3614	&	3251	&	8587	&	1548	&	2216	&	3903	\\
				SMAD4	&	4650	&	4301	&	17900	&	-	&	-	&	-	\\
				BAX	&	4720	&	5048	&	2166	&	-	&	-	&	-	\\
				TGFBR2	&	6130	&	5796	&	6804	&	-	&	-	&	-	\\
				BRAF	&	6159	&	5915	&	2680	&	-	&	-	&	-	\\
				CTNNB1	&	6945	&	7162	&	14361	&	-	&	-	&	-	\\
				APC	&	6951	&	6683	&	8829	&	-	&	-	&	-	\\
				MYO1B	&	7919	&	8008	&	17026	&	-	&	-	&	-	\\
				MSH6	&	7999	&	8007	&	3834	&	-	&	-	&	-	\\
				NRAS	&	8054	&	7870	&	8141	&	-	&	-	&	-	\\
				CDK8	&	8454	&	8273	&	4790	&	-	&	-	&	-	\\
				CASP8	&	8877	&	8916	&	3823	&	-	&	-	&	-	\\
				MAP7	&	8939	&	8795	&	4847	&	-	&	-	&	-	\\
				PTPN12	&	9056	&	8933	&	4631	&	-	&	-	&	-	\\
				ACVR2A	&	9130	&	9057	&	14557	&	-	&	-	&	-	\\
				FAM123B	&	10163	&	10183	&	10360	&	-	&	-	&	-	\\
				MIER3	&	11637	&	11911	&	10526	&	-	&	-	&	-	\\
				MLH3	&	14408	&	14406	&	8321	&	-	&	-	&	-	\\
				KIAA1804	&	15606	&	16195	&	17573	&	-	&	-	&	-	\\
				SMAD2	&	16641	&	16776	&	13170	&	-	&	-	&	-	\\
				MSH3	&	17890	&	17353	&	13725	&	-	&	-	&	-	\\
				TCERG1	&	18664	&	19121	&	16361	&	-	&	-	&	-	\\
				\hline
			\end{tabular}
		}
	\end{center}
	\caption{Genes with known involvement in colon cancer and their rankings}
	\label{tab:ccgr}
\end{table*}

\begin{table*}[th!]
	\begin{center}
		\resizebox{\linewidth}{!}{
			\begin{tabular}{|ccccccc|}
				\hline
				Method & CMNC-OBF-JP & BD & t-test & SKBS & BPM & BayesReg(logit,$L_1$) \\  \hline
				$\#$ known markers selected($N=1000$) & 6 & 6 & 3 & 3 & 1 & 6 \\  
				p-value($N=1000$) & $1.3 \times 10^{-3}$ & $1.3 \times 10^{-3}$ & $0.09$ & $0.09$ & $0.52$ & $1.3 \times 10^{-3}$ \\  \hline
				$\#$ known markers selected($N=2000$) & 9 & 9 & 5 & 6 & 8 & 7 \\  
				p-value($N=2000$) & $1.5 \times 10^{-3}$ & $1.5 \times 10^{-3}$ & $0.1216$ & $0.05$ & $5.5 \times 10^{-3}$ & $0.018$ \\  \hline
				$\#$ known markers selected($N=5000$) & 17 & 16 & 16 & 15 & 15 & 15 \\  
				p-value($N=5000$) & $5.5 \times 10^{-4}$ & $1.8 \times 10^{-3}$ & $1.8 \times 10^{-3}$ & $5 \times 10^{-3}$ & $5 \times 10^{-3}$ & $5 \times 10^{-3}$ \\  
				\hline
			\end{tabular}
		}
	\end{center}
	\caption{Number of known biomarkers that rank below fixed thresholds and the associated p-values}
	\label{tab:ccrpv}
\end{table*}

\begin{table*}[t]
	\begin{center}
		\resizebox{\linewidth}{!}{
			\begin{tabular}{|cccccccccc|}
				\hline
				Rank & Gene & $\pi^*(f)$ & t-test p-value & $\pi^{S*}(f)$  & Rank & Gene & $\pi^*(f)$ & t-test p-value & $\pi^{S*}(f)$  \\ \hline
				1	&	CPNE4		&	$>0.999$	&	$	0.208	$	&	$	0.98	$	&	11	&	ADIPOQ	&	$>0.999$	&	$	1.6 \times 10^{-4}	$	&	$	0.979	$	\\
				2	&	GAGE1,4-7,12		&	$>0.999$	&	$	0.003	$	&	$	0.977	$	&	12	&	CT45A1-6	&	$>0.999$	&	$	0.051	$	&	$	0.979	$	\\
				3	&	GAGE1,2,4-8,12		&	$>0.999$	&	$	0.012	$	&	$	0.977	$	&	13	&	HORMAD1	&	$>0.999$	&	$	0.243	$	&	$	0.981	$	\\
				4	&	S100A7		&	$>0.999$	&	$	0.099	$	&	$	0.997	$	&	14	&	FGG	&	$>0.999$	&	$	0.222	$	&	$	0.981	$	\\
				5	&	EPHA7		&	$>0.999$	&	$	0.502	$	&	$	0.992	$	&	15	&	AHSG	&	$>0.999$	&	$	0.171	$	&	$	0.981	$	\\
				6	&	MAGEA4		&	$>0.999$	&	$	0.002	$	&	$	0.978	$	&	16	&	MMP8	&	$>0.999$	&	$	3.2\times 10^{-7}	$	&	$	0.98	$	\\
				7	&	LOC100133920,286297		&	$>0.999$	&	$	0.194	$	&	$	0.978	$	&	17	&	TF	&	$>0.999$	&	$	0.026	$	&	$	0.981	$	\\
				8	&	GPM6A		&	$>0.999$	&	$	1.8 \times 10^{-8}	$	&	$	0.979	$	&	18	&	LOC100507186	&	$>0.999$	&	$	0.479	$	&	$	0.986	$	\\
				9	&	SLC2A2		&	$>0.999$	&	$	0.329	$	&	$	0.978	$	&	19	&	SLC14A1	&	$>0.999$	&	$	0.002	$	&	$	0.978	$	\\
				10	&	CEACAM5		&	$>0.999$	&	$	3.3 \times 10^{-5}	$	&	$	0.973	$	&	20	&	FLJ37786	&	$>0.999$	&	$	0.435	$	&	$	0.993	$	\\
				\hline
			\end{tabular}
		}
	\end{center}
	\caption{Top 20 genes selected by OBF and their statistics}
	\label{tab:t20obf}
\end{table*}

\begin{table*}[t]
	\begin{center}
		\resizebox{\linewidth}{!}{
			\begin{tabular}{|cccccccccc|}
				\hline
				Rank & Gene & $\pi^*(f)$ & t-test p-value & $\pi^{S*}(f)$  & Rank & Gene & $\pi^*(f)$ & t-test p-value & $\pi^{S*}(f)$  \\ \hline
				1	&	RAC3		&	$	0.999	$	&	$	1.67 \times 10^{-8}	$	&	$	0.999	$	&	11	&	PSTPIP2	&	$	0.979	$	&	$	1.24\times 10^{-6}	$	&	$	0.978	$	\\
				2	&	FOXD4,4L1		&	$	>0.999	$	&	$	1.82\times 10^{-8}	$	&	$	0.979	$	&	12	&	AMACR	&	$	0.999	$	&	$	1.40\times 10^{-6}	$	&	$	0.979	$	\\
				3	&	SMARCA4		&	$	>0.999	$	&	$	6.66 \times 10^{-8}	$	&	$	0.982	$	&	13	&	KCNT2	&	$	0.991	$	&	$	1.45\times 10^{-6}	$	&	$	0.98	$	\\
				4	&	LOC401463		&	$	0.986	$	&	$	3.02\times 10^{-7}	$	&	$	0.676	$	&	14	&	ZBTB8B	&	$	0.999	$	&	$	2.52\times 10^{-6}	$	&	$	0.981	$	\\
				5	&	XAB2		&	$	>0.999	$	&	$	3.16\times 10^{-7}	$	&	$	0.98	$	&	15	&	SLC26A4	&	$	0.944	$	&	$	2.52\times 10^{-6}	$	&	$	0.999	$	\\
				6	&	KATNAL2		&	$	0.988	$	&	$	3.45\times 10^{-7}	$	&	$	0.993	$	&	16	&	CETP	&	$	0.985	$	&	$	2.86\times 10^{-6}	$	&	$	0.982	$	\\
				7	&	FUT9		&	$	0.999	$	&	$	6.76\times 10^{-7}	$	&	$	0.979	$	&	17	&	TIMM17A	&	$	>0.999	$	&	$	2.92\times 10^{-6}	$	&	$	0.984	$	\\
				8	&	TFB1M		&	$	0.959	$	&	$	1.02\times 10^{-6}	$	&	$	>0.999	$	&	18	&	PSEN1	&	$	0.948	$	&	$	3.64\times 10^{-6}	$	&	$	0.923	$	\\
				9	&	CD58		&	$	0.968	$	&	$	1.12\times 10^{-6}	$	&	$	>0.999	$	&	19	&	WDR91	&	$	0.991	$	&	$	5.03 \times 10^{-6}	$	&	$	0.981	$	\\
				10	&	EMG1		&	$	0.974	$	&	$	1.17 \times 10^{-6}	$	&	$	0.981	$	&	20	&	TAF2	&	$	0.857	$	&	$	5.59\times 10^{-6}	$	&	$	0.99	$	\\
				\hline
			\end{tabular}
		}
	\end{center}
	\caption{Top 20 genes selected by t-test and their statistics}
	\label{tab:t20ttest}
\end{table*}

\begin{figure}
	\centering
	\includegraphics[width=0.6\textwidth]{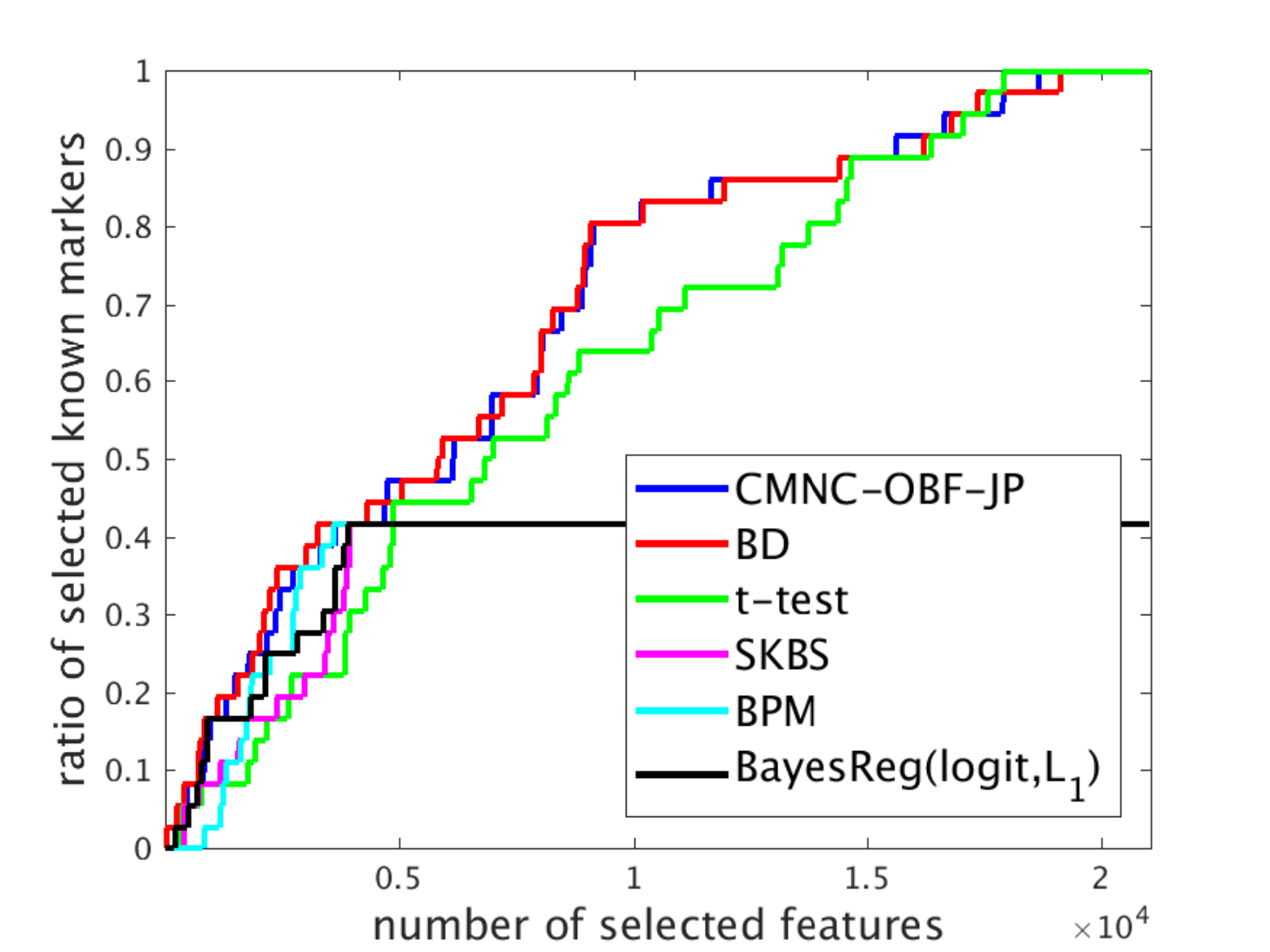}
	\caption{Proportion of selected known colon cancer markers listed in Table \ref{tab:ccgr} versus total number of selected features.}\label{fig:BAROC}
\end{figure}

\begin{figure}
	\centering
	\includegraphics[width=0.75\textwidth]{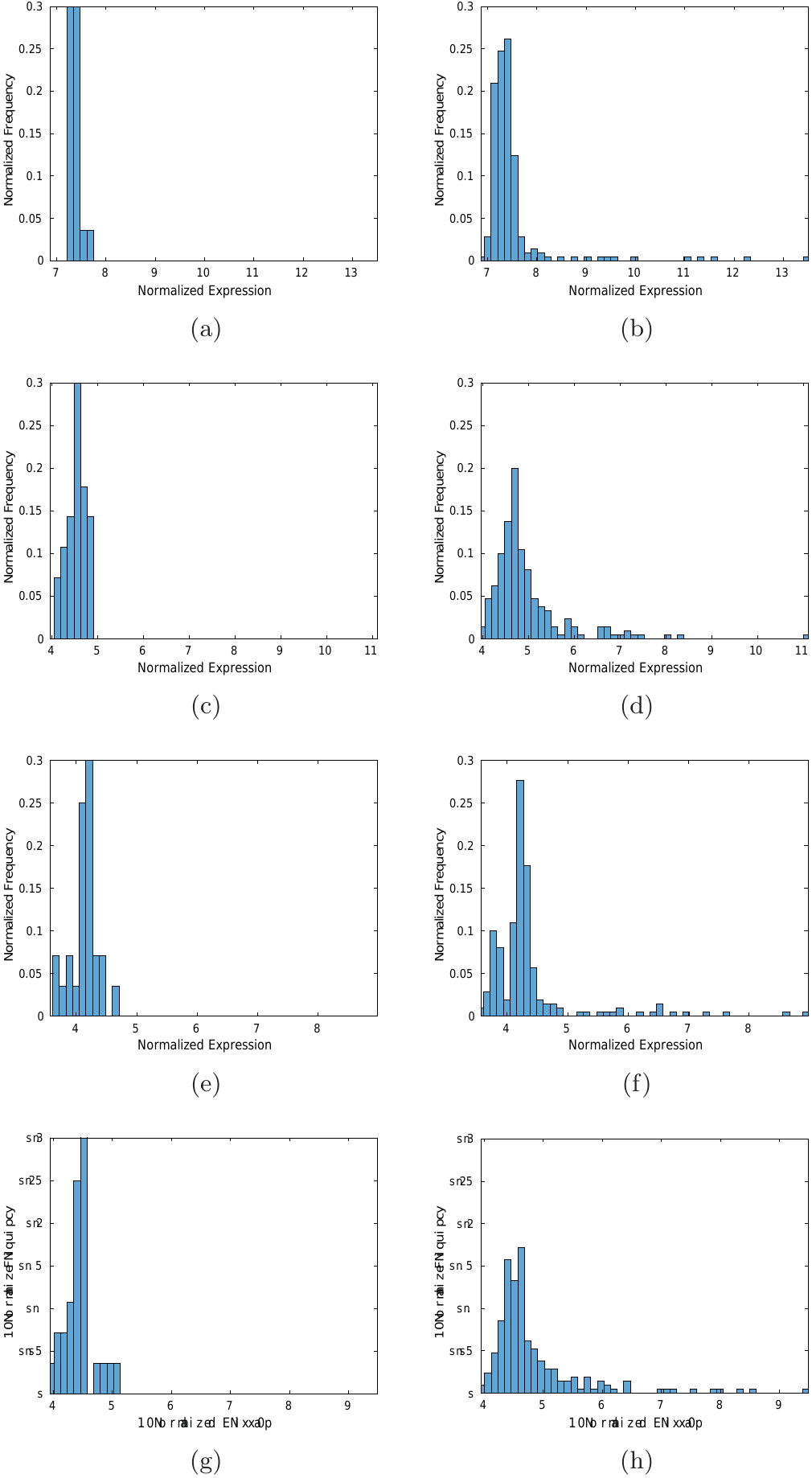}
%
%
%
	\caption{Histograms of normalized gene expression values for stage 1 (class 0) versus stage 2-4 (class 1) colon cancer patients. (a) GAGE1,4-7,12, class 0; (b) GAGE1,4-7,12, class 1; (c) GPM6A, class 0; (d) GPM6A, class 1; (e) SLC14A1, class 0; (f) SLC14A1, class 1; (g) MMP8, class 0; (h) MMP8, class 1.}\label{fig:hist}
\end{figure}

We perform an enrichment analysis of the top $2,000$ CMNC-OBF-JP genes using an over-representation test in PANTHER~\citep{mi2009panther, mi2017panther}. PANTHER recognized $321$ genes in the gene list, and tests over $163$ pathways. The top $10$ pathways are listed in Table~\ref{tab:colon_panther}. Bounding FDR by $0.05$, the top 4 pathways are significant. Meanwhile, many of the top PANTHER pathways, for instance, the cadherin signaling pathway \citep{avizienyte2002src,pena2005cadherin}, the WNT signaling pathway \citep{bienz2000linking}, the plasminogen activating cascade \citep{Ganesh4065,baker2003plasminogen}, and blood coagulation \citep{PMID:2694422}, have been shown to be involved or affected in colon cancer.

\begin{table*}[t]
	\begin{center}
		\resizebox{\linewidth}{!}{
			\begin{tabular}{|cccccc|}
				\hline
				Rank	&	Pathway Name	&	Pathway Size	&	Num. Genes Selected	&	p-value	&	Adjusted p-value	\\ \hline
				1	&	Cadherin signaling pathway	&	158	&	62	&	1.50E-17	&	2.45E-15	\\
				2	&	Wnt signaling pathway	&	312	&	79	&	4.71E-13	&	2.56E-11	\\
				3	&	Plasminogen activating cascade	&	18	&	8	&	1.10E-03	&	3.58E-02	\\
				4	&	Alzheimer disease-presenilin pathway	&	123	&	25	&	9.74E-04	&	3.97E-02	\\
				5	&	Integrin signalling pathway	&	190	&	33	&	2.35E-03	&	5.48E-02	\\
				6	&	CCKR signaling map	&	174	&	31	&	2.31E-03	&	6.27E-02	\\
				7	&	Blood coagulation	&	46	&	12	&	3.29E-03	&	6.70E-02	\\
				8	&	Huntington disease	&	145	&	24	&	1.26E-02	&	2.29E-01	\\
				9	&	Cytoskeletal regulation by Rho GTPase	&	84	&	16	&	1.77E-02	&	2.88E-01	\\
				10	&	5HT4 type receptor mediated signaling pathway	&	33	&	8	&	2.13E-02	&	3.16E-01	\\
				\hline
			\end{tabular}
		}
	\end{center}
	\caption{Top $10$ over-represented pathways for early- versus late-stage colon cancer}
	\label{tab:colon_panther}
\end{table*}

\section{Discussion on the Running Time of Penalized Regression Models}

\label{sec:RTC}

Here we discuss how cross validation and stability selection affect the running time of GLMs with LASSO and elastic net. 
10 fold cross validation partitions the data into 10 folds, uses 9 folds for training and the remaining fold for testing, and considers all 10 combinations of leaving 1 out fold for testing. This process is performed several times; we implement 10 Monte Carlo repetitions in the colon cancer example. For a penalized regression model, this is done for all candidate penalty terms, and the penalty that results in best overall prediction is selected. Let $T$ be the average time to train the penalized regression model for one penalty term.  Note that 10 fold cross validation uses $90\%$ of data for training. The computational complexity of regularized regression models depends on the training algorithm as well as the data (sample size and number of predictors). In general, it can be $O(n)$ or $O(n^2)$ depending on the settings used. A more detailed discussion on the computational complexity of penalized GLMs can be found in \cite{minka2003comparison} and \cite{hastie2015statistical}. Assuming linear and quadratic complexity with respect to sample size, the running time of cross validation for one penalty value is approximately $10 \times 0.9 \times M \times T=9 M T$ and $10 \times 0.9^2 \times M \times T=8.1 M T$, respectively, where $M$ is the number of Monte Carlo repetitions. Either way, the running time of cross validation is $>8MT$. Assuming $L$ candidate penalties, we expect the total running time of cross validation to be more than $8LMT$. While $T$ might be reasonable, $8LMT$ might be large.

Now we discuss the computational complexity of the stability selection method of \cite{meinshausen2010stability}. Recall that we generate 100 subsamples of the data, where each subsample uses $90\%$ of the points in each class. Following a similar argument to cross validation, the running time of stability selection should be more than $80LT$.

For the colon cancer data, using the penalty term $\lambda=0.01$, LASSO and elastic net took about 25.8 and 29.2 seconds to run, respectively, and for $\lambda=0.1$ LASSO and elastic net took about 1.1 and 2.7 seconds to run, respectively. The average running time is approximately $T=14$ seconds. Recall that we tested 40 different penalty terms and cross validation used 10 Monte Carlo repetitions, i.e., $L=40$ and $M=10$. For each of the penalty families (LASSO and elastic net), cross validation and stability selection each took about 15 hours, which is consistent with our estimates. The excessive running times of cross validation and stability selection prevented us from implementing them in the synthetic simulations of Section 6 of the main manuscript.

\section{Discussion on Regression and Classification Objectives for Feature Selection}

Variable selection has been extensively studied for the purpose of regression \citep{o2009review,park_bayesian_2008,xu2015bayesian,baragatti2011bayesian}, and has been discussed in detail in the main manuscript. Here we provide several examples on the types of features such methods can detect, what can be expected of them in small-sample high-dimensional biomarker discovery settings, and the observations made in the bioinformatics community regarding the applicability of such models for biomarker discovery. In particular, we focus on why regression and classification based objectives may not be suitable for biomarker discovery applications, where one desires to find \textit{all} features with distributional differences.

A classical method such as t-test can detect differences in means, but cannot detect differences in variances. For example, if a feature has densities similar to Fig. \ref{fig:ba_reg}(a) in two classes, t-test cannot detect it. The SLC14A1 gene studied in the real data example is an example of such genes. In contrast, methods based on regression and classification objectives can detect this mode of distributional difference. Furthermore, they can detect other second order distributional differences, such as the one in Fig. \ref{fig:ba_reg}(b), where each color denotes the joint distribution between two features in a given class. Neither t-test nor OBF can detect such biomarkers, but feature selection methods proposed in the authors' previous work, such as 2MNC-Robust and POFAC \citep{dalton2017heuristic,fdin}, are specifically designed to address this issue. 2MNC-Robust and POFAC are more computationally intensive than OBF, but typically not as intensive as methods based on regression and classification. In the main manuscript we have focused on consistency properties of OBF due to its simplicity and having a closed form solution for $\pi^*(f)$. This also lays the foundation for consistency analysis of POFAC and 2MNC-Robust, which will be discussed in future work.

Suppose two features have class-conditioned densities similar to Fig. \ref{fig:ba_reg}(c). Both t-test and OBF can easily detect these two biomarkers, thanks to their differences in their means, but objectives based on regression and classification performance struggle. Both features are extremely correlated, and dependencies are similar in both classes, which is a typical co-expression setting in ``omics'' studies. In this case,  given the value of one feature, say feature 1, one can easily and accurately estimate the value of the other feature, (feature 2 in this example). Therefore, a classifier considering both features does not improve performance very much relative to a classifier using only one of these two features. Now, given some fixed observed data, feature selection schemes based on different classification rules, or even simply different permutations of the same data, may end up selecting different features, e.g., one pipeline may report only feature 1 while another may report only feature 2. Furthermore, suppose another study is performed and another sample is observed. Due to many factors, for instance noise and experimental conditions, a classifier might select feature 1 given sample 1, and feature 2 given sample 2. Thereby, reported feature sets are not reproducible. These are some undesirable properties of classification based feature selection algorithms in biomarker discovery applications, which have been discussed in many reviews \citep{ilyin_biomarker_2004,diamandis_cancer_2010,saeys_review_2007}. Furthermore, (a) the peaking phenomenon may result in reporting very few biomarkers \citep{sima2008peaking}, and (b) classification error estimates may suggest sets far from the set with minimal Bayes error \citep{sima2006should}. In the real data example, we observed that GLMs with LASSO and elastic net penalties miss many important biomarkers, although they may achieve good prediction performance.

\begin{figure}[t!]
\includegraphics[width=\textwidth]{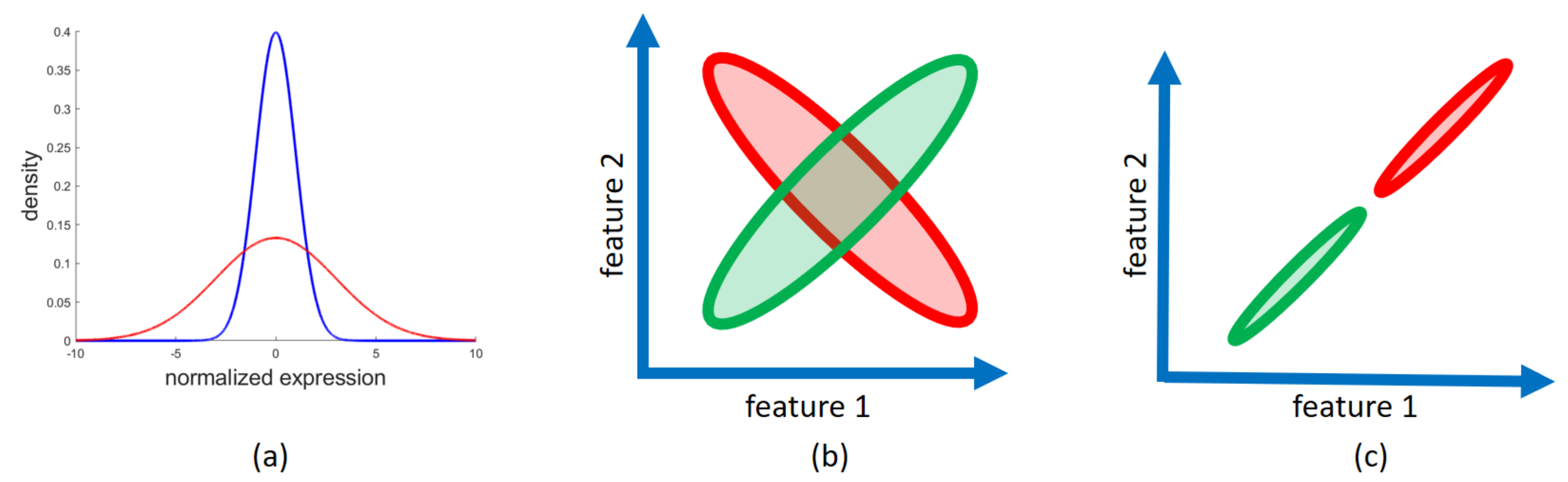}
	\caption{Example bivariate distributions. (a) Gaussian features with different variances; (b) Jointly Gaussian features with different class-conditioned covariances; (c) Jointly Gaussian features with different class-conditioned means.}\label{fig:ba_reg}
\end{figure}


\section{Justification for Using Improper Priors}

\label{sec:justify}

Here we show that $\pi^*(f)$ under an improper prior is the limit of a sequence of proper priors, which are constructed by truncating the improper prior. Recall that for each feature $f \in F$, $p(\theta^f_0)$, $p(\theta^f_1)$, and $p(\theta^f)$ are the priors on $\theta^f_0$ assuming $f \in \bar{G}$, on $\theta^f_1$ assuming $f \in \bar{G}$, and on $\theta^f$ assuming $f \in \bar{B}$, respectively. Note that $p(\theta^f_y)$ is described by hyperparameters $s^f_y$, $\kappa^f_y$, $\nu^f_y$, and $m^f_y$, and relative weights $A^f_y$ and $B^f_y$. Similarly, $p(\theta^f)$ is described by hyperparameters $s^f$, $\kappa^f$, $\nu^f$, and $m^f$, and relative weights $A^f$ and $B^f$. We assume these hyperparameters result in improper normal-inverse-Wishart priors such that their area, i.e., their integrals, are infinite. In particular, we assume the integrals of $p(\sigma_y^f)$, $p(\mu_y^f|\sigma_y^f)$, $p(\sigma^f)$, and $p(\mu^f|\sigma^f)$ are all infinite. A similar demonstration can be made
for improper priors where any combination of $p(\sigma_y^f)$, $p(\mu_y^f | \sigma_y^f)$, $p(\sigma^f)$ and $p(\mu^f | \sigma^f)$ are improper.

We first describe the truncation process of $p(\theta^f_y)$. The truncation process of $p(\theta^f)$ is similar. Recall that  $\theta^f_y=[\mu^f_y,\sigma^f_y]$. Consider the proper prior $p_{K_y,M_y}(\theta^f_y)$ such that
\begin{align}
\label{eq:gyvar}
p_{K_y,M_y}(\sigma^f_y) \propto \begin{cases}
(\sigma^f_y)^{-0.5 (\kappa^f_y+2)} e^{-0.5 s^f_y/\sigma^f_y} &\quad \text{if\ } 0<\sigma^f_y<K_y, \\
0 &\quad \text{otherwise,}
\end{cases}
\end{align}
and
\begin{align}
\label{eq:gymean}
p_{K_y,M_y}(\mu^f_y|\sigma^f_y) \propto \begin{cases}
(\sigma^f_y)^{-0.5} e^{-0.5 \nu^f_y(\mu^f_y-m^f_y)^2/\sigma^f_y} &\quad \text{if\ } |\mu^f_y|<M_y(\sigma^f_y), \\
0 &\quad \text{otherwise,}
\end{cases}
\end{align}
where $K_y$ is a positive constant and $M_y(\sigma_y^f)$ is a positive function of $\sigma_y^f$. In general, $K_y$ and $M_y(\sigma^f_y)$ may depend on the feature $f$ as well, i.e., we should write $K^f_y$ and $M^f_y(\sigma^f_y)$. However, we have dropped the superscript $f$ here to avoid cluttered notation. Let $U^f_y$ be the normalization constant of $p_{K_y,M_y}(\sigma^f_y)$, and $V^f_y(\sigma^f_y)$ be the normalization constant of $p_{K_y,M_y}(\mu^f_y|\sigma^f_y)$. In general, $V^f_y$ may depend on $\sigma^f_y$ and we have explicitly included this dependence for the sake of being complete; however, we will later choose $M_y(\sigma^f_y)$ such that $V^f_y$ does not depend on $\sigma^f_y$. Now consider the proper prior $p_{K,M}(\theta^f)$ such that
\begin{align}
\label{eq:bvar}
p_{K,M}(\sigma^f) \propto \begin{cases}
(\sigma^f)^{-0.5 (\kappa^f+2)} e^{-0.5 s^f/\sigma^f} &\quad \text{if\ } 0<\sigma^f<K, \\
0 &\quad \text{otherwise,}
\end{cases}
\end{align}
and
\begin{align}
\label{eq:bmean}
p_{K,M}(\mu^f|\sigma^f) \propto \begin{cases}
(\sigma^f)^{-0.5} e^{-0.5 \nu^f(\mu^f-m^f)^2/\sigma^f} &\quad \text{if\ } |\mu^f|<M(\sigma^f), \\
0 &\quad \text{otherwise.}
\end{cases}
\end{align}
Again, $M$ may depend on $\sigma^f$, and we are using $K$ and $M(\sigma^f)$ instead of $K^f$ and $M^f(\sigma^f)$ to avoid cluttered notation. Let $U^f$ be the normalization constant of $p_{K,M}(\sigma^f)$, and $V^f(\sigma^f)$ be the normalization constant of $p_{K,M}(\mu^f|\sigma^f)$.

Let $A(r): [0,\ \infty) \to [0,\ \infty)$ be a strictly increasing function such that $A(0)=0$ and $\lim_{r \to \infty} A(r)=\infty$. Now, for each $r>0$, choose $K_0$, $K_1$, $M_0(\sigma_0^f)$, $M_1(\sigma_1^f)$, $K$, and $M(\sigma^f)$ such that
\begin{align}
\label{eq:Ufy}
U^f_y&=A^f_y/A(r), \\
\label{eq:Vfy}
V^f_y&=B^f_y/A(r), \\
\label{eq:Uf}
U^f&=A^f/A(r)^2, \\
\label{eq:Vf}
V^f&=B^f/A(r)^2.
\end{align}
Note that $M_y(\sigma^f_y)$ and $M(\sigma^f)$ are chosen such that $V^f_y$ and $V^f$ do not depend on $\sigma^f_y$ and $\sigma^f$, respectively. In order to satisfy \eqref{eq:Ufy} through \eqref{eq:Vf}, we need to choose $K_y$, $M_y(\sigma^f_y)$, $K$, and $M(\sigma^f)$ such that
\begin{align}
&\int_0^{K_y} (\sigma^f_y)^{-0.5 (\kappa^f_y+2)} e^{-0.5 s^f_y/\sigma^f_y} d\sigma^f_y = A(r)/A^f_y, \nonumber \\
&\int_{-M_y(\sigma^f_y)}^{M_y(\sigma^f_y)} (\sigma^f_y)^{-0.5} e^{-0.5 \nu^f_y(\mu^f_y-m^f_y)^2/\sigma^f_y} d\mu^f_y = A(r)/B^f_y, \nonumber \\
&\int_0^{K}(\sigma^f)^{-0.5 (\kappa^f+2)} e^{-0.5 s^f/\sigma^f} d\sigma^f = A(r)^2/A^f, \nonumber \\
&\int_{-M(\sigma^f)}^{M(\sigma^f)} (\sigma^f)^{-0.5} e^{-0.5 \nu^f(\mu^f-m^f)^2/\sigma^f} d\mu^f = A(r)^2/B^f. \nonumber
\end{align}
This is doable since $p(\sigma^f_y)$, $p(\mu^f_y|\sigma^f_y)$, $p(\sigma^f)$, and $p(\mu^f|\sigma^f)$ are all improper. For each $r>0$ let
\begin{align}
\label{eq:hr_def}
h^r(f) = \frac{\pi(f)}{1-\pi(f)} \times
\frac
{\int p_{K_0,M_0}(\theta^f_0) p(S^f_0|\theta^f_0) d\theta^f_0 \int  p_{K_1,M_1}(\theta^f_1) p(S^f_1|\theta^f_1) d\theta^f_1}
{\int p_{K,M}(\theta^f) p(S^f|\theta^f) d\theta^f}.
\end{align} 
In order to compute the integrals in \eqref{eq:hr_def}, we first need to integrate with respect to the means, and then the variances. Observe that
\begin{align}
\label{eq:hrf_calc}
h^r(f)&=
\frac{\pi(f)}{1-\pi(f)} \times
\frac{A^f_0 B^f_0 A^f_1 B^f_1}{A^f B^f} \times I_0(f) I_1(f) (I(f))^{-1},
\end{align}
where
\begin{align}
I_y(f) &=
\int_0^{K_y} \int_{-M_y(\sigma^f_y)}^{M_y(\sigma^f_y)}
(\sigma^f_y)^{-0.5 (\kappa^f_y+3)} e^{-0.5 (s^f_y+\nu^f_y(\mu^f_y-m^f_y)^2)/\sigma^f_y}
p(S^f_y|\sigma^f_y,\mu^f_y) d\mu^f_y d\sigma^f_y,
\nonumber \\ I(f)&=
\int_0^{K} \int_{-M(\sigma^f)}^{M(\sigma^f)} 
(\sigma^f)^{-0.5 (\kappa^f+3)} e^{-0.5 (s^f+\nu^f(\mu^f-m^f)^2)/\sigma^f}
p(S^f|\sigma^f,\mu^f)
d\mu^f d\sigma^f, \nonumber
\end{align}
and we have used \eqref{eq:gyvar}, \eqref{eq:gymean}, \eqref{eq:bvar}, and \eqref{eq:bmean} to explicitly write out the priors $p_{K_y,M_y}(\theta^f_y)$ and $p_{K,M}(\theta^f)$.
Observe that $A(r) \to \infty$ as $r \to \infty$, and hence $K_y,M_y,K,M \to \infty$ as $r \to \infty$. If
\begin{align}
\int p(\theta^f_0) p(S^f_0|\theta^f_0) d\theta^f_0, \quad
\int p(\theta^f_1) p(S^f_1|\theta^f_1) d\theta^f_1, \quad
\text{ and }
\int p(\theta^f) p(S^f|\theta^f) d\theta^f \nonumber
\end{align} 
exist, then using the Monotone Convergence Theorem (MCT) we have that
\begin{align}
\lim_{r \to \infty} \frac{A^f_0 B^f_0 A^f_1 B^f_1}{A^f B^f} \times I_0(f) I_1(f) (I(f))^{-1} = &\prod_{y=0,1} \int p(\theta^f_y) p(S^f_y|\theta^f_y) d\theta^f_y \nonumber \\  &\times \left( \int p(\theta^f) p(S^f|\theta^f) d\theta^f \right)^{-1}. \nonumber
\end{align}
Therefore, we have
\begin{align}
\label{eq:hrlim}
\lim_{r \to \infty} h^r(f)=h(f),
\end{align}
where $h(f)$ is computed using (2.18) of the main manuscript. Here, we have chosen the radii of the truncated priors so that $h^r(f)$ converges to the desired value as $r \to \infty$.  This is not the only solution; indeed, any functional form for the normalization constants $U_y^f$, $V_y^f$, $U^f$ and $V^f$ such that 
\begin{equation}
\lim_{r\to\infty} \frac{U_0^f V_0^f U_1^f V_1^f}{U^f V^f} = \frac{A_0^f B_0^f A_1^f B_1^f}{A^f B^f} = L^f
\end{equation}
would guarantee that~\eqref{eq:hrlim} holds.  In addition, under improper priors $L^f$ is an arbitrary constant set by the user, and for any $L^f$ there exists a sequence of proper priors for which~\eqref{eq:hrlim} holds.  It is even possible to design the radii of the truncated priors such that $h^r(f) \to 0$ ($\pi^*(f) \to 0$) or $h^r(f) \to \infty$ ($\pi^*(f) \to 1$) as $r \to \infty$ for a fixed sample, $S$.  While the current analysis does not help in justifying the use of any particular value for $L^f$, it does give an interpretation for $L^f$.  In particular, from~\eqref{eq:Ufy} through~\eqref{eq:Vf} we see that the constants $A_y^f$, $B_y^f$, $A^f$ and $B^f$ that make up $L^f$ control the relative normalization constants needed in the truncated improper priors, which in turn control the relative rate that the radii of the truncated priors increase. These radii must increase at rates such that the ratio of the area under the un-normalized truncated $p_{K,M}(\theta^f)$ and the product of the areas under the un-normalized truncated $p_{K_0,M_0}(\theta^f_0)$ and $p_{K_1,M_1}(\theta^f_1)$ is equivalent to (or converges to) the constant $L^f$.

\section{A Note on the Jeffreys-Lindley Paradox}

\label{sec:noteJLP}

The Jeffreys-Lindley paradox is encountered when Bayesian methods using improper priors and frequentist approaches yield statistics that motivate different actions (accepting or rejecting the null hypothesis) for some observed data. This is a major concern for using improper priors in practice, and has been an active topic of debate for the past three decades \citep{robert1993note,robert2014jeffreys,berger1987testing,spanos2013should}

Following a classical example from \cite{robert1993note} and \cite{berger1987testing} for illustration of this paradox, we consider testing the mean of a Gaussian population with unit variance. Assuming the population follows the density $N(\theta,1)$, the null ($H_0$) is $\theta=\theta_0$ and the alternative ($H_1$) is $\theta \neq \theta_0$. Suppose the prior probability of $H_0$ is $p_0$, and consider a Gaussian prior on $\theta$ under $H_1$, i.e., $p(\theta)=N(0,\sigma)$, where $\sigma$ is the variance of the Gaussian prior. For $\theta_0=0$, and given an observation $x$ from the population, we have that the posterior probability of $H_0$ is \citep{robert1993note}
\begin{align}
\left( 1+ \frac{1-p_0}{p_0}
\frac{e^{-x^2/(2\sigma+2)}}{e^{-x^2/2}}
\frac{1}{\sqrt{\sigma+1}}
\right)^{-1},
\end{align}
which goes to one as $\sigma$ goes to infinity for any given $x$ and $p_0$. As $\sigma$ goes to infinity, the prior on $\theta$ under $H_1$ assigns less weight to each fixed neighborhood of $\theta_0$. Thereby, the alternative hypothesis would have smaller posterior probability for each fixed $\theta_0$ and  $x$ as $\sigma$ goes to infinity. Larger values of $\sigma$ correspond to less-informative priors; thus a non-informative setup seems to produce meaningless results here. Further discussion on this paradox for point null hypothesis tests, and possible remedies, are provided in \cite{robert1993note} and \cite{robert2014jeffreys}. One approach to this problem is to avoid non-informative priors.  This may be feasible if, for example, $\theta_0$ has been chosen by the experimenter because it has some special meaning for the problem at hand.  Alternatively, one may allow $p_0$ to depend on $\sigma$ so that the posterior on $H_0$ does not converge to extreme values as $\sigma$ goes to infinity.  In general, one must take care when using non-informative priors, especially when setting normalization constants associated with improper prior densities.  However, there is currently no universal agreement on precisely how this should be done.


Here we consider sequences of proper priors with increasing variances for our feature selection problem, and study how they affect $\pi^*(f)$. In other words, we study how the choice of $p(\theta^f_0)$, $p(\theta^f_1)$, and $p(\theta^f)$ affect $\pi^*(f)$ as we make them less informative, i.e., increase their variance and make them more flat. Suppose $p(\theta^f_y)$ follows a proper normal-inverse-Wishart prior with hyperparameters $s^f_y$, $\kappa^f_y$, $m^f_y$, and $\nu^f_y$, and $p(\theta^f)$ follows a proper normal-inverse-Wishart prior with hyperparameters $s^f$, $\kappa^f$, $m^f$, and $\nu^f$. Furthermore, fix sample $S$ with $n_0,n_1>1$. From (2.18) of the main manuscript observe that we have
\begin{align}
h(f) &=\frac{\pi(f)}{1-\pi(f)} \times
\left( \frac{\nu^f_0 \nu^f_1 \nu^{f*}}{\nu^f \nu^{f*}_0 \nu^{f*}_1} \right)^{0.5} \times
\frac{\Gamma(0.5 \kappa^f) \Gamma(0.5 \kappa^{f*}_0) \Gamma(0.5 \kappa^{f*}_1)}{\Gamma(0.5 \kappa^f_0) \Gamma(0.5 \kappa^f_1) \Gamma(0.5 \kappa^{f*})} \nonumber \\ &\times
\frac{(s^f_0)^{0.5 \kappa^f_0}(s^f_1)^{0.5 \kappa^f_1}}{(s^f)^{0.5 \kappa^f}} \times
\frac{(s^{f*})^{0.5 \kappa^{f*}}}{(s^{f*}_0)^{0.5 \kappa^{f*}_0}(s^{f*}_1)^{0.5 \kappa^{f*}_1}}.
\end{align}
Assume $s^f_y$, $\kappa^f_y$, $m^f_y$, $s^f$, $\kappa^f$, and $m^f$ are all fixed, and vary $\nu^f_y$ and $\nu^f$. Observe that
\begin{enumerate}
\item{If $\nu^f_0=\nu^f_1=\nu^f=\nu$, then $\lim_{\nu \to 0} h(f)=0$ and hence $\lim_{\nu \to 0} \pi^*(f)=0$.}
\item{If $\nu^f_0=c_0 \nu$, $\nu^f_1=c_1 \nu$, and $\nu^f= c_b \nu^2$ for some $c_0,c_1,c_b \in (0, \infty)$, then $\lim_{\nu \to 0} h(f) = c$ for some $c \in (0,\infty)$, and $\lim_{\nu \to 0} \pi^*(f)=c/(1+c)$.}
\item{If $\nu^f_0=\nu^f_1=\nu$ and $\nu^f=\nu^3$, then $\lim_{\nu \to 0} h(f)=\infty$ and hence $\lim_{\nu \to 0} \pi^*(f)=1$.}
\end{enumerate}
Depending on how $\nu_0^f$, $\nu_1^f$ and $\nu^f$ go to zero relative to each other, different behaviors may occur. Such behaviors of likelihood ratios are discussed in \cite{villa2017mathematics}. This is in contrast to the classical example provided above, where no mater how we increase the variance of the prior on $\theta_0$, the posterior probability of $H_0$ goes to 1. Here, in order to obtain $h(f) \not \to 0,\infty$, we should select the $\nu$'s such that $\nu^f_0 \nu^f_1 / \nu^f$ (which is a component of $L^f$) approaches a positive constant in the limit. Assuming $f$ is a good feature we have two degrees of freedom for choosing the hyperparameters, i.e., $\nu^f_0$ and $\nu^f_1$, and assuming $f$ is bad we only have one hyperparameter to tune, i.e., $\nu^f$. Heuristically speaking, we can visualize this as the following: since under the assumption that $f$ is a good feature we have two degrees of freedom and under the assumption that $f$ is a bad feature we only have one degree of freedom, we should properly select $\nu^f$ compared with $\nu^f_0$ and $\nu^f_1$ to promote the same amount of ``uncertainty" in the priors and avoid $\pi^*(f) \to 0 \text{ or } 1$. Note that for $\nu^f_0=\nu^f_1=\nu^f=0$ we do not need to specify $m^f_0$, $m^f_1$, and $m^f$.

Now consider the case where $\kappa^f_y$ and $\kappa^f$ go to zero, and $s^f_y$, $s^f$, $m^f_y$, $m^f$, $\nu^f_y$, and $\nu^f$ are  fixed. In this case, to avoid $h(f)$ converging to zero or infinity, we require a sequence of $\kappa$'s such that $\Gamma(0.5\kappa^f)/(\Gamma(0.5\kappa_0^f)\Gamma(0.5\kappa_1^f))$ converges to a positive constant in the limit.  It is well known (e.g., via Taylor series at zero) that $\Gamma(x)$ is asymptotically equal to $1/x$ as $x$ goes to zero.  Thus, we equivalently require $\kappa_0^f \kappa_1^f / \kappa^f$ to converge to a positive constant. This is similar to the situation above where we let the $\nu$'s go to zero. For instance, we may set $\kappa^f_0=c'_0 \kappa$, $\kappa^f_1=c'_1 \kappa$, and $\kappa^f=c'_b \kappa$ for some $c^{\prime}_0, c^{\prime}_1, c^{\prime}_b \in (0, \infty)$ and let $\kappa$ go to zero to get $h(f) \not \to 0,\infty$ for the observed sample $S$. 

In Section \ref{sec:justify} we considered sequences of proper priors built by truncating improper priors.  Similarly, here we have considered sequences of proper priors where we let the $\nu$'s and $\kappa$'s go to zero.  In all cases, $\pi^*(f)$ converges to a positive constant only when the sequence of parameters being tweaked (the radii of truncated priors, the $\nu$'s, and the $\kappa$'s) are chosen carefully in combination.  The critical issue always boils down to how $L^f$ should be selected. While these analyses help with setting $L^f$ in practice, currently it is being subjectively chosen by the user under improper priors and the choice of $L^f$ remains a topic for future work.  The consistency proof in Section 5 of the main manuscript offers some reassurance that the data will eventually win out if $L^f$ is selected poorly.  Perhaps a natural choice for a non-informative prior is $c_0=c_1=c_b=c'_0=c'_1=c'_b=1$, which results in $L^f=(2\pi)^{-0.5}$ (to cancel out the $(2\pi)^{0.5}$ in (2.18) of the main manuscript). This is close to the value $L^f=0.1$ used in the simulations for Section 6 of the main manuscript and in the real data analysis performed in Section~\ref{sec:RMD}.

Now assume $L^f$ is such that the non-informative prior is semi-proper by Definition 3 of the main manuscript. Using Lemma \ref{lemma:p_trend} in Section \ref{sec:lemmas}
 and (2.18) of the main manuscript we see that under JP for $n$ large enough
\begin{align}
\label{eq:h_approx}
h(f) \approx \frac{c L^f \pi(f)}{n(1-\pi(f))} \left(
\frac{(c^f)^{\kappa^{f*}}}
{(c_0^f)^{\kappa_0^{f*}} (c_1^f)^{\kappa_1^{f*}}}
\right)^{0.5},
\end{align}
for some $c>0$. Assuming $f$ is an independent unambiguous bad feature, using (5.16) and (5.35) of the main manuscript, we have that for $n$ large enough with probability 1
\begin{align}
\label{eq:logbound}
\left( \frac{(c^f)^{\kappa^{f*}}}
{(c_0^f)^{\kappa_0^{f*}} (c_1^f)^{\kappa_1^{f*}}}
\right)^{0.5} < L_b (\log n)^c
\end{align}
for some $L_b,c>0$. Hence, for a bad feature we expect $h(f)$ (which is always positive) to decay at least as fast as $(\log n)^c/n$ for some $c>0$ as $n$ goes to infinity. For independent unambiguous good features, by Lemma \ref{lemma:good_variance} in Section \ref{sec:lemmas} we see that
\begin{equation}
\left(\frac{(c^f)^{\kappa^{f*}}}
{(c_0^f)^{\kappa_0^{f*}} (c_1^f)^{\kappa_1^{f*}}}\right)^{0.5} > L_g R^n
\end{equation}
where $L_g>0$ and $R > 1$.  Thus, the right-hand side of \eqref{eq:h_approx} grows at least exponentially fast. Therefore, reasonable values of $L^f$ should give satisfactory performance. Note that extremely large values of $L^f$ result in large $\pi^*(f)$'s and hence more false alarms under the MR objective, and extremely small values of $L^f$ result in missing more features under MR. 

Another reasonable choice for $L^f$ can be $\left( \sqrt{2 \pi} n\right) ^{-1}$, which is again semi-proper with $p=-1$. For a bad feature $f$, using \eqref{eq:h_approx} and \eqref{eq:logbound}, we have that for $n$ large enough with probability 1, $\pi^*(f)<(\log n)^c/n^2$ for some $c>0$. Thereby, not only do we have $\pi^*(f) \to 0$ as sample size increases with probability 1 when $f$ is bad, but also $\sum_{n=1}^{\infty} \pi^*(f) < \infty$. Note that $L^f=\left( \sqrt{2 \pi} n\right) ^{-1}$ has the slowest polynomial decay with respect to $n$ (assuming the power of $n$ is an integer)  while satisfying $\sum_{n=1}^{\infty} \pi^*(f) < \infty$ for bad features. Furthermore, $h(f)$ would still grow exponentially fast for good features.

Finally, when CMNC is used and $L^f$ is independent of the feature index $f$, the exact value of $L^f$ does not matter. However, the choice of $D$, i.e., the number of features to select, is now determined by the user. It is desirable to use improper priors when little or no reliable information is available, or a proper normal-inverse-Wishart prior might not adequately describe the prior information. The robust performance of OBF under improper priors is studied in detail in \cite{TCBB}.

We close by showing that OBF-JP provides a feature ranking equivalent to that of a frequentist statistic used to test equality in two Gaussian populations. Under JP,
\begin{align}
\label{eq:hjp}
h(f)&= \frac{\pi(f)}{1-\pi(f)} L^f
\bigg(\frac{2 \pi (n_0+n_1)}{n_0 n_1}\bigg)^{0.5} 
\frac{\Gamma(0.5 n_0) \Gamma(0.5 n_1)}
{\Gamma(0.5 (n_0+n_1))} \nonumber \\ &\times
\frac{(n-1)^{0.5n}}{(n_0-1)^{0.5n_0} (n_1-1)^{0.5n_1}}
\frac{(\hat{\sigma}^f)^{0.5n}}{(\hat{\sigma}^f_0)^{{0.5n_0}} (\hat{\sigma}^f_1)^{0.5n_1}}.
\end{align}

Posing the feature selection problem as a hypothesis test for feature $f \in F$, under the null ($H_0$) we have $\theta^f_0=\theta^f_1$, i.e., $f$ is a bad feature, and under the alternative ($H_1$) we have $\theta^f_0 \neq \theta^f_1$, i.e., $f$ is a good feature. Under this setup, the test is not a point null hypothesis test as $\theta^f_0$ and $\theta^f_1$ are not fixed to have a specific value; however, whatever they are, they must be equal. The hypothesis test of \cite{NP_OBF_JP_freq} for comparing two Gaussian populations suggests using the test statistic
\begin{align}
\lambda_{n_0,n_1}(f)=\frac{(\check{\sigma}^f_0)^{0.5 n_0} (\check{\sigma}^f_1)^{0.5 n_1}}{ (\check{\sigma}^f)^{0.5n}},
\end{align}
where $\check{\sigma}^f_y$ and $\check{\sigma}^f$ are biased variance estimates of $f$ in class $y$ and both classes considered together, which divide the sum of squares by $n_y$ and $n$, respectively. Therefore,
\begin{align}
\label{eq:cvy}
\check{\sigma}^f_y&=\frac{n_y-1}{n_y} \hat{\sigma}^f_y, \\
\label{eq:cv}
\check{\sigma}^f&=\frac{n-1}{n} \hat{\sigma}^f.
\end{align}
Given $\lambda_{n_0,n_1}(f)$, p-values can be found using the method of \cite{zhang2012exact}. Therefore, if $\pi(f)$ and $L^f$ do not depend on the feature index $f$, using \eqref{eq:hjp}, \eqref{eq:cvy}, and \eqref{eq:cv} we see $h(f) \propto 1/\lambda_{n_0,n_1}(f)$, and hence OBF-JP and the p-values computed using $\lambda_{n_0,n_1}(f)$ provide the same feature ranking for a given sample $S$.


\end{document}